\documentclass[sigconf]{acmart}
\AtBeginDocument{%
  }

\setcopyright{acmlicensed}
\copyrightyear{2018}
\acmYear{2018}
\acmConference[Conference acronym 'XX]{Make sure to enter the correct
  conference title from your rights confirmation email}{June 03--05,
  2018}{Woodstock, NY}
\acmDOI{XXXXXXX.XXXXXXX}
\acmISBN{978-1-4503-XXXX-X/2018/06}

\usepackage{subcaption}
\usepackage[utf8]{inputenc} 
\usepackage[T1]{fontenc}    
\usepackage{hyperref}       
\usepackage{url}            
\usepackage{booktabs}       
\usepackage{amsfonts}       
\usepackage{nicefrac}       
\usepackage{microtype}      
\usepackage{xcolor} 
\usepackage{makecell}
\usepackage{amsmath}
\usepackage{graphicx}
\usepackage{wrapfig}
\usepackage{enumitem}
\usepackage{multirow}
\usepackage{xcolor}
\usepackage[table]{xcolor}
\usepackage{CJKutf8}
\usepackage{amsthm}
\usepackage{bm}
\usepackage{pifont}
\usepackage{caption}
\usepackage{mathtools}

\begin{document}

\title{U-shaped Multi-granularity Learning for Vision-Language Models}

\author{Biao Chen}
\affiliation{%
  \institution{University of Electronic Science and Technology of China}
  \city{Chengdu}
  \country{China}
}
\affiliation{%
  \institution{Hithink Research, Hangzhou, China}
  \country{}
}
\email{chenbiao2@myhexin.com}

\author{Yunqian Yu}
\affiliation{%
  \institution{University of Electronic Science and Technology of China}
  \city{Chengdu}
  \country{China}
}
\email{yuyunqian2022@163.com}

\author{Xiangxu Zhao}
\affiliation{%
  \institution{University of Electronic Science and Technology of China}
  \city{Chengdu}
  \country{China}
}
\email{xxzhao@std.uestc.edu.cn}

\author{Zhongshu Chen}
\affiliation{%
  \institution{University of Electronic Science and Technology of China}
  \city{Chengdu}
  \country{China}
}
\email{zschen@std.uestc.edu.cn}

\author{Mengmeng Jing}
\affiliation{%
  \institution{University of Electronic Science and Technology of China}
  \city{Chengdu}
  \country{China}
}
\email{jingmeng1992@gmail.com}

\author{Lin Zuo}
\authornote{Corresponding author.}
\affiliation{%
  \institution{University of Electronic Science and Technology of China}
  \city{Chengdu}
  \country{China}
}
\email{linzuo@uestc.edu.cn}

\renewcommand{\shortauthors}{Biao Chen, Yunqian Yu, Xiangxu Zhao, Zhongshu Chen, Mengmeng Jing, and Lin Zuo}

\begin{abstract}
The prompt learning paradigm for vision-language models is effective yet faces a granularity dilemma: global prompts lack fine-grained semantic awareness, while local prompts ignore contextual associations, limiting cross-task generalization. This dilemma exists in dense prediction tasks.
Inspired by U-Net, which unifies multi-level representations across granularities, we propose UPrompt, a U-shaped multi-granularity prompt learning framework for vision-language models.
Similar to how U-Net integrates fine and coarse features through symmetric encoder-decoder pathways with cross-level connections, UPrompt constructs parallel multi-granularity representations in both visual and textual modalities, where coarse-to-fine cascaded enhancement propagates global context to refine local details, while fine-to-coarse hierarchical supervision ensures semantic consistency across scales. 
Extensive experiments on 17 benchmarks validate our effectiveness. UPrompt outperforms MAMET and VPKE by 4.1 and 7.3 rSum on MSCOCO, surpasses CoCoA-Mix by 5.09\% in base-to-novel generalization, while maintaining competitive performance with minimal overhead (coarse-grained) and matching PSRC with 1/3 cost (medium-grained). Our code is available at https://github.com/JustCoolPig/UPrompt.
\end{abstract}




\maketitle

\section{Introduction}
\label{sec:intro}
Prompt learning has emerged as an effective paradigm for Vision-Language Model (VLM) adaptation by optimizing learnable prompt tokens~\cite{zhou2022learning, khattak2023maple, chen2026dropout}. 
However, existing methods~\cite{yu2025visual, yao2023visual, CoPrompt, chen2025chain} mainly optimize at a single, fixed granularity, creating an inherent trade-off between capturing broad context and preserving fine-grained visual details~\cite{guo2025dvit, yu2026instruction}. Global prompting strategies cannot encode local features for fine-grained reasoning: CoOp~\cite{zhou2022learning}, using a single global prompt, underperforms fine-grained TAP~\cite{dingtree} by 11.31\% on the FGVCAircraft dataset. Conversely, finely structured prompts struggle to integrate sufficient global context or model cross-region compositional relationships, as illustrated in Fig.~\ref{motivation}. This granularity bottleneck significantly constrains adaptation performance and generalization across diverse vision-language tasks.

In view of this problem, we turn to hierarchical architectures that have demonstrated remarkable success in dense prediction tasks, exemplified by U-Net~\cite{ronneberger2015u}. 
U-Net enables effective multi-scale modeling through its symmetric encoder-decoder design and skip connections, jointly preserving high-level semantic context and fine-grained local details. However, transferring these principles to prompt learning is non-trivial. A fundamental paradigm gap exists: U-Net operates on spatially structured pixel grids~\cite{williams2023unified}, whereas prompt learning functions in an abstract embedding space~\cite{huang2025adaptive,li2024relationship}, where textual prompts lack inherent geometric structure. This raises two core challenges: constructing meaningful multi-granularity representations that embody semantic hierarchy in both modalities, and establishing bidirectional information flow between granularities to ensure cross-level consistency and complementary learning.

\begin{figure*}
\vspace{-8pt}
    \centering    \includegraphics[width=0.80\linewidth]{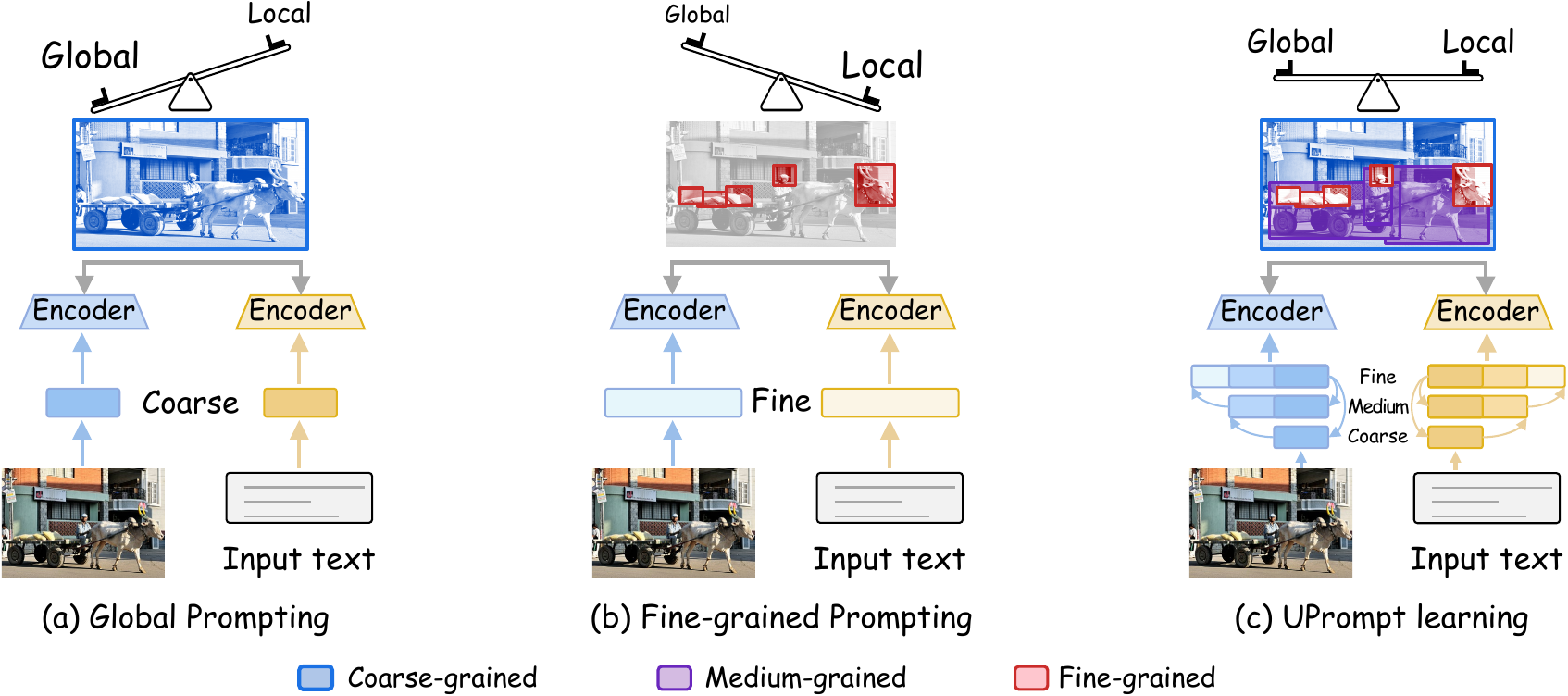}
    \vspace{-10pt}
    \caption{\textbf{The granularity trade-off in prompt learning.} (a) Global prompting captures broad context but lacks fine-grained details, while (b) fine-grained prompting preserves local features but loses global information. (c) UPrompt learning addresses the trade-off through multi-granularity hierarchical modeling, achieving both global understanding and local precision.}
    \label{motivation}
    \vspace{-10pt}
\end{figure*}
To address these challenges, we propose UPrompt learning, a U-shaped multi-granularity prompt learning framework that introduces a structured hierarchy into vision-language adaptation, as shown in Fig.~\ref{motivation}(c). Inspired by U-Net's multi-scale feature fusion mechanisms, such as skip connections~\cite{ronneberger2015u}, UPrompt constructs parallel granularity pathways in vision and language, using progressive spatial pooling for images and iterative semantic enrichment for text.
Moreover, the framework incorporates a bidirectional connection mechanism analogous to U-Net's skip connections: a coarse-to-fine cascaded mechanism that injects global context into fine-grained features using cross-granularity attention, and fine-to-coarse hierarchical supervision that distributes semantic knowledge from the finest to the coarsest levels via distillation. This ensures not only enhanced representational capacity at each level, but also consistent semantics across granularities.
Comprehensive evaluation across diverse benchmarks validates the effectiveness of our approach. UPrompt bridges hierarchical representation learning principles from U-Net to prompt-based VLMs, establishing a unified framework that resolves the limitation of single granularity through bidirectional information flow. This architecture provides flexible multi-granularity alignment while ensuring semantic consistency across representation scales, offering a principled solution for vision-language adaptation. Our contributions are as follows:
\begin{itemize}
[leftmargin=*,itemsep=1.2ex,parsep=0ex,label=$\bullet$]
  \item We introduce UPrompt, a U-Net-inspired framework for prompt learning that leverages hierarchical multi-granularity representations across vision-language modalities to overcome single-scale adaptation limitations.
  \item We introduce bidirectional connection, establishing bidirectional information flow across multi-granularity hierarchies. Coarse-to-fine enhancement injects global context into fine-grained representations for improved local modeling, while Fine-to-Coarse supervision leverages finest-level alignment to regularize coarser granularities, ensuring semantic consistency.
  \item Experiments on 17 benchmarks demonstrate UPrompt's superiority in cross-modal retrieval, few-shot classification, base-to-novel generalization, and out-of-distribution scenarios, while its hierarchical design enables flexible performance-efficiency trade-offs.
\end{itemize}



\section{Related Work}
\label{related_work}

\noindent\textbf{Prompt Learning in VLMs. }
CoOp~\cite{zhou2022learning} introduced prompt learning to CLIP~\cite{radford2021learning}, which was later extended to both visual and textual modalities~\cite{khattak2023maple,cho2023distribution}. To overcome single-global-prompt limitations, subsequent methods explore multi-granularity representations. GalLoP~\cite{lafon2024gallop} uses dual prompts for global and local features, TAP~\cite{dingtree} derives diverse prompts from attribute trees, and SurPL~\cite{liusurrogate} generates dynamic features. HiCroPL~\cite{zheng2025hierarchical} injects prompts at multiple levels, while SPTR~\cite{cui2025similarity} employs diverse fixed prompt ensembles. However, these methods largely treat granularities as independent modules fused only at the final stage, without unified modeling of cross-scale dependencies and information flow. Our U-Net-inspired framework instead introduces bidirectional connections for progressive integration and semantic consistency across the hierarchy.

\noindent\textbf{Hierarchical Representation.}
Multi-scale feature fusion, established by FPN~\cite{lin2017feature} and U-Net~\cite{ronneberger2015u}, is fundamental to visual understanding and has been extended to diverse architectures. UNet++~\cite{zhou2019unet++} uses nested skip connections to bridge encoder-decoder semantic gaps, while GraphFPN~\cite{zhao2021graphfpn} constructs data-dependent feature pyramids for graph neural networks. HGFormer~\cite{ding2023hgformer} performs part-whole grouping in Vision Transformers, and MSVMamba~\cite{shi2024multi} incorporates hierarchical design into State Space Models, confirming its continued relevance. VLMs must capture both coarse context and fine details, requiring corresponding textual representations at each visual granularity. However, text lacks inherent spatial structure, making direct hierarchical construction difficult. We therefore build and align multi-granularity hierarchies across modalities.

\begin{figure*}
    \centering    \includegraphics[width=1\linewidth]{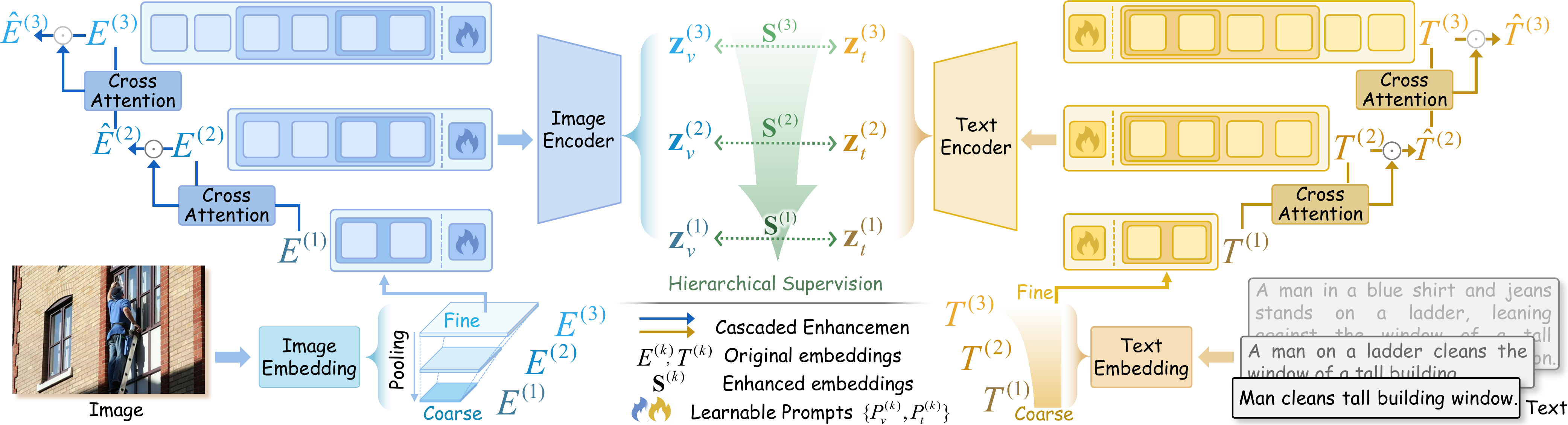}
    \vspace{-10pt}
    \caption{\textbf{Method overview of UPrompt.} UPrompt learning constructs hierarchical vision-language alignment via multi-granularity pathways with learnable prompts. Bidirectional connection operates through coarse-to-fine cascaded enhancement that injects global context into fine-grained embeddings via cross-attention, and fine-to-coarse hierarchical supervision that guides coarser levels using finest-grained representations.}
    \label{pipeline}
    \vspace{-10pt}
\end{figure*}
\noindent\textbf{Cross-Level Interaction.}
Hierarchical representations commonly follow coarse-to-fine and fine-to-coarse paradigms. Coarse-to-fine methods, such as Stacked Hourglass Networks~\cite{newell2016stacked} and RefineNet~\cite{lin2017refinenet}, propagate global context to refine local details. Fine-to-coarse methods, including GroupViT~\cite{xu2022groupvit} and HVQ~\cite{lu2023hierarchical}, aggregate low-level features to maintain high-level consistency. NeRD-Rain~\cite{chen2024bidirectional} further enables bidirectional flow, refining features with coarser context while enriching them with finer details. Inspired by these strategies, we address the isolated optimization of hierarchical prompts in VLMs through bidirectional connections that enable progressive integration and semantic consistency across the hierarchy.

\section{Methodology}\label{method}
\subsection{Preliminaries}\label{Preliminaries}
\noindent\textbf{CLIP.}
Contrastive Language-Image Pre-training (CLIP)~\cite{radford2021learning} uses an image encoder $\mathcal{F}$ and a text encoder $\mathcal{G}$ to map image $x$ and text $t$ into a shared space: $\mathbf{z}_v = \mathrm{norm}(\mathcal{F}(x))$, $\mathbf{z}_t = \mathrm{norm}(\mathcal{G}(t))$. A symmetric contrastive loss aligns matched pairs and separates mismatched ones, enabling zero-shot classification and cross-modal retrieval. With temperature $\tau$, the probability is:
\begin{equation}
p(k|x) = \frac{\exp(\mathrm{sim}(\mathbf{z}_v, \mathbf{z}_{t,k}) / \tau)}{\sum_{j} \exp(\mathrm{sim}(\mathbf{z}_v, \mathbf{z}_{t,j}) / \tau)}.
\end{equation}
\noindent\textbf{Prompt learning methodology.}
Prompt learning efficiently optimizes VLMs by incorporating learnable prompt tokens instead of full fine-tuning approaches~\cite{zhou2022learning,khattak2023maple}. The textual and visual input token sequences at transformer layer $i$ are formally defined as:
$T_{input}^{(i)} = \{t_{bos}, P_{t}^{(i)}, T_{embed}, t_{eos}\}$ and 
$V_{input}^{(i)} = \{v_{cls}, E_{patch}, P_{v}^{(i)}\}$, 
where $P_{t}^{(i)} = \{p_t^1, p_t^2, ..., p_t^\eta\}$ and $P_{v}^{(i)} = \{p_v^1, p_v^2, ..., p_v^M\}$ are learnable prompt vectors with dimensions $\mathbb{R}^{\eta}$ and $\mathbb{R}^{M}$ respectively.

\noindent\textbf{U-shaped architecture.}
U-shaped networks (e.g., U-Net~\cite{ronneberger2015u}) consist of a symmetric encoder-decoder design with skip connections between corresponding layers. Let the network have $L$ levels; the encoder at level $i$ outputs features $\mathbf{h}^{(i)}$, and the decoder at level $i$ fuses them with the upsampled features from level $i+1$:
\begin{equation}
\tilde{\mathbf{h}}^{(i)} = \phi^{(i)}\big(\mathbf{h}^{(i)}, \mathrm{up}(\tilde{\mathbf{h}}^{(i+1)})\big),
\end{equation}
where $\phi^{(i)}(\cdot)$ is a cross-level fusion operator and $\mathrm{up}(\cdot)$ denotes upsampling. This enables multi-level information propagation, maintaining global context while preserving fine details.

\subsection{U-Shaped Multi-Granularity Prompting}
\noindent\textbf{UPrompt Learning Paradigm.}
To address the limitation of trade-off between global context and local details in single-granularity prompt learning, we draw inspiration from U-Net's hierarchical processing where $\tilde{\mathbf{h}}^{(i)} = \phi^{(i)}(\mathbf{h}^{(i)}, \mathrm{up}(\tilde{\mathbf{h}}^{(i+1)}))$ fuses multi-level features, maintaining both global context and fine-grained details across scales. Existing multi-level methods (e.g., TAP, HiCroPL) treat granularities as independent modules or operate within network depths, lacking cross-scale dependencies across semantic hierarchies. We propose U-shaped multi-granularity prompting, dubbed as UPrompt learning, that constructs parallel hierarchical semantic structures with explicit cross-granularity information flow across modalities. We extend CLIP encoders $\mathcal{F}$ and $\mathcal{G}$ to multi-granularity versions ${\mathcal{F}^{(k)}, \mathcal{G}^{(k)}}_{k=1}^K$ where $k \in [1, K]$ spans coarsest to finest granularities.


For visual modality, we construct nested patch hierarchies through progressive downsampling pooling in embedding space. The input image is first processed by CLIP's patch embedding layer to extract the finest-grained patch tokens $E_{patch}^{(K)}$. Coarser representations are derived via recursive pooling: $E_{patch}^{(k)} = \text{Pool}^{(k)}(E_{patch}^{(k+1)})$ for $k = K-1, \ldots, 1$, ensuring $|E_{patch}^{(k)}| < |E_{patch}^{(k+1)}|$. At each granularity level $k$, we concatenate the pooled patch embeddings with learnable prompts $P_v^{(k)} \in \mathbb{R}^{M \times d}$ and feed them into CLIP's vision encoder to obtain visual features: $\mathbf{z}_v^{(k)} = \mathcal{F}^{(k)}([E_{patch}^{(k)}; P_v^{(k)}])$.

For textual modality, we construct semantic hierarchies via progressive enrichment:
\begin{align}
T_{embed}^{(1)} &= \Phi_{\text{abstract}}(t), \\
T_{embed}^{(k)} &= T_{embed}^{(k-1)} \oplus \Phi_{\text{refine}}^{(k)}(t), \quad k = 2, \ldots, K,
\end{align}
where $\Phi_{\text{abstract}}(\cdot)$ extracts core semantics, $\Phi_{\text{refine}}^{(k)}(\cdot)$ generates granularity-specific elaborations, and $\oplus$ is semantic expansion ensuring nesting $T_{embed}^{(k)} \subset T_{embed}^{(k+1)}$. Text representations integrate prompts $P_t^{(k)} \in \mathbb{R}^{\eta \times d}$, i.e., $\mathbf{z}_t^{(k)} = \mathcal{G}^{(k)}([P_t^{(k)}; T_{embed}^{(k)}])$. Operators $\Phi_{\text{abstract}}(\cdot)$ and $\Phi_{\text{refine}}^{(k)}(\cdot)$ are instantiated via LLMs with specific prompts for multi-level text generation (implementation details in Sec.~\ref{experiments}).

Cross-modal alignment at each granularity $k$ is achieved through similarity $\mathbf{S}^{(k)}(x,t) = \frac{\mathbf{z}_v^{(k)} \cdot \mathbf{z}_t^{(k)}}{\|\mathbf{z}_v^{(k)}\| \|\mathbf{z}_t^{(k)}\|}$. This U-shaped architecture enables multi-granularity vision-language alignment through hierarchical prompt learning across symmetric pathways, with visual branch providing spatial representations and textual branch offering semantic specifications.

\subsection{Bidirectional Connection for UPrompt Learning}
Simple granularity stacking in UPrompt learning lacks inter-granularity interaction, causing fine-grained context deficiency and coarse-grained optimization inconsistency that limit multi-granularity representation. To address these challenges, we propose bidirectional connection for UPrompt learning, which employs coarse-to-fine cascaded enhancement during forward propagation and fine-to-coarse hierarchical supervision during backward optimization (Fig.~\ref{pipeline}).

\noindent\textbf{Coarse-to-Fine Cascaded Enhancement (CE).}
To address context deficiency where fine-grained embeddings lack global contextual guidance for modeling local information relationships, we propose cascaded enhancement that injects coarse-grained contextual information into finer embeddings.
For embeddings $X^{(k)} \in \{E_{patch}^{(k)}, T_{embed}^{(k)}\}$ at granularity level $k$, the enhancement operation:
\begin{equation}
\hat{X}^{(k)} = X^{(k)} \odot \mathcal{A}(X^{(k)}, \hat{X}^{(k-1)}),
\label{eq.5}
\end{equation}
where $\hat{X}$ are enhanced embeddings, $\odot$ is element-wise product, $\mathcal{A}(\cdot, \cdot)$ is cross-granularity attention:
\begin{equation}
\begin{aligned}
\mathcal{A}(X^{(k)}, \hat{X}^{(k-1)}) = &\text{softmax}\left(\frac{X^{(k)}\mathbf{W}_q (\hat{X}^{(k-1)}\mathbf{W}_k)^{\top}}{\sqrt{d}}\right) \\
& \cdot \hat{X}^{(k-1)}\mathbf{W}_v,
\end{aligned}
\label{eq.6}
\end{equation}
where $\mathbf{W}_{q,k,v} \in \mathbb{R}^{d \times d}$ are query, key, and value projection matrices, $d$ is the embedding dimension. Enhanced embeddings are fed into encoders to obtain $\mathbf{z}_v^{(k)} = \mathcal{F}^{(k)}([\hat{E}_{patch}^{(k)}; P_v^{(k)}])$ and $\mathbf{z}_t^{(k)} = \mathcal{G}^{(k)}([P_t^{(k)}; \hat{T}_{embed}^{(k)}])$. Fine-grained embeddings can extract contextually relevant information from global representations, enhancing local information modeling with global contextual guidance.

\begin{proposition}[CE Directional Alignment Effect]\label{prop:ce}
Let $\hat{X}^{(k)}$ be the fine-grained representation at level $k$ enhanced by coarse-to-fine cascaded enhancement (CE, Eq.~(\ref{eq.5})-(\ref{eq.6})), which leverages contextual guidance from the coarser representation $\hat{X}^{(k-1)}$. Let $X^{(k)}$ be its unenhanced counterpart. Under the mild assumption that the coarse context is informative, CE provably strengthens alignment between fine-grained features and coarse-grained guidance in expectation:
\begin{equation}
\mathbb{E}\left[ \frac{\langle \hat{X}^{(k)}, \hat{X}^{(k-1)} \rangle}{\|\hat{X}^{(k)}\| \|\hat{X}^{(k-1)}\|} \right] \ge \mathbb{E}\left[ \frac{\langle X^{(k)}, \hat{X}^{(k-1)} \rangle}{\|X^{(k)}\| \|\hat{X}^{(k-1)}\|} \right].
\end{equation}
\end{proposition}
\noindent\textit{Proof.} Cascaded enhancement injects coarse contextual guidance into fine representations via cross-attention, provably improving directional alignment in expectation. See Appendix~\ref{proof_of_prop:ce}.

\noindent\textbf{Fine-to-Coarse Hierarchical Supervision (HS).}
To address optimization inconsistency from semantic drift at coarse granularities, we propose fine-to-coarse hierarchical supervision using the superior alignment of finest-grained features. The finest-level features $(\mathbf{z}_v^{(K)}, \mathbf{z}_t^{(K)})$ achieve optimal cross-modal correspondence through rich representational capacity, serving as teacher signals for coarser levels.
The finest-level cross-modal similarity matrix $\mathbf{S}^{(K)}$ provides the teacher distribution:
\begin{equation}
\mathbf{S}^{(K)}_{ij} = \frac{\mathbf{z}_{v,i}^{(K)} \cdot \mathbf{z}_{t,j}^{(K)}}{\|\mathbf{z}_{v,i}^{(K)}\| \|\mathbf{z}_{t,j}^{(K)}\|}.
\label{equation:similarity}
\end{equation}
All coarser levels are supervised via knowledge distillation from the detached finest-level representations, preventing degradation of the teaching signal by coarse-grained semantic drift:
\begin{equation}
\begin{aligned}
\mathcal{L}_{\text{guide}} = \frac{1}{K-1}\sum_{k=1}^{K-1} \mathbb{E}_{(i,j)}\Big[ D_{\text{KL}}\Big(&\text{softmax}\left(\mathbf{S}^{(k)}_{i,:} / \tau_d\right) \\
&\| \text{softmax}\left(\mathbf{S}^{(K)}_{i,:} / \tau_d\right)\Big)\Big].
\end{aligned}
\label{equation:hs}
\end{equation}
where $\tau_d$ is the distillation temperature. Detaching $\mathbf{S}^{(K)}$ prevents gradients from coarse-level training flowing back to the finest layer, ensuring that fine-grained misalignment does not corrupt coarse representations. Concurrently, CE injects global contextual guidance into fine-grained features, enabling self-correction across granularities. This hierarchical supervision enforces semantic consistency across all granularity levels, preventing coarse-grained drift while keeping the complementary nature of multi-granularity representations enables more effective cascaded enhancement.
\begin{proposition}[HS Consistency and Substitutability]\label{prop:hs}
Let $S^{(k)}$ and $S^{(K)}$ be similarity matrices from Eq.~(\ref{equation:similarity}), and define $p^{(k)}_{\tau_d}(j|i) = \text{softmax}(S^{(k)}_{i,:}/\tau_d)_j$ and $q^{(K)}_{\tau_d}(j|i) = \text{softmax}(S^{(K)}_{i,:}/\tau_d)_j$ where teacher $q^{(K)}$ is detached as in Eq.~(\ref{equation:hs}). Assuming HS aligns coarse-grained distributions with fine-grained teachers, HS bounds semantic drift and enables performance-preserving coarse inference:
\begin{equation}
\begin{aligned}
&\mathbb{E}_{(x,t),i}\!\left[\mathrm{KL}\!\left(q^{(K)}_{\tau_d}(\cdot|i) \,\|\, p^{(k)}_{\tau_d}(\cdot|i)\right)\right] \le \varepsilon \\
&\implies \mathbb{E}_{(x,t),i}\!\bigl[\bigl|\Phi\!\left(p^{(k)}_{\tau_d}(\cdot|i)\right)-\Phi\!\left(q^{(K)}_{\tau_d}(\cdot|i)\right)\bigr|\bigr] \le L\sqrt{\varepsilon/2},
\end{aligned}
\end{equation}
for any $L$-Lipschitz functional $\Phi$ w.r.t. total variation distance. The detach operation ensures gradient isolation: $\partial L_{guide}/\partial z^{(K)} = 0$.
\end{proposition}
\noindent\textit{Proof.} Hierarchical supervision constrains KL divergence between coarse and fine distributions, yielding bounded substitutability via Pinsker's inequality. See Appendix~\ref{proof_of_prop:hs}.

\noindent\textbf{Overall Objective.}
The UPrompt learning framework combines contrastive alignment loss across all $K$ granularity levels with hierarchical supervision for cross-modal alignment and inter-granularity consistency. Contrastive losses are averaged for stable optimization:
\begin{equation}
\begin{split}
\mathcal{L}_{\text{UPrompt}} &= \mathcal{L}_{\text{guide}} + \\
& \frac{1}{K}\sum_{k=1}^{K} \mathbb{E}_{(x,t)}\left[-\log 
\frac{\exp(\text{sim}(\mathbf{z}_{v}^{(k)}, \mathbf{z}_{t}^{(k)}) / 
\tau)}{\sum_{t'} \exp(\text{sim}(\mathbf{z}_{v}^{(k)}, 
\mathbf{z}_{t'}^{(k)}) / \tau)}\right].
\end{split}
\end{equation}
During inference, we can flexibly select granularity levels based on performance-efficiency trade-offs. We default to finest-grained features $(\mathbf{z}_v^{(K)}, \mathbf{z}_t^{(K)})$ for optimal performance. When prioritizing efficiency, coarser levels offer reduced token requirements and lower costs while preserving semantic consistency via our fine-to-coarse hierarchical supervision that prevents coarse-grained drift.

\begin{table*}[t]
  \centering
  \caption{\textbf{Base-to-novel generalization.}  Bold values indicate the best results. HM: Harmonic Mean.}
  \vspace{-8pt}
  \setlength{\tabcolsep}{6.4pt}
  \renewcommand{\arraystretch}{0.90}
  \scalebox{0.95}{
    \begin{tabular}{lccc|ccc|ccc|ccc}
    \toprule
    \multirow{2}[4]{*}{Method} & \multicolumn{3}{c}{Average} & \multicolumn{3}{c}{ImageNet} & \multicolumn{3}{c}{Caltech101} & \multicolumn{3}{c}{OxfordPets} \\
\cmidrule{2-13}          & Base  & Novel & HM    & Base  & Novel & HM    & Base  & Novel & HM    & Base  & Novel & HM \\
    \midrule
    CoOp$_{\text{(IJCV'22)}}$ & 82.69  & 63.22  & 71.66  & 76.47  & 67.88  & 71.92  & 96.00  & 89.81  & 93.73  & 93.67  & 95.29  & 94.47  \\
    PSRC$_{\text{(ICCV'23)}}$ & 84.26  & 76.10  & 79.97  & 77.60  & 70.73  & 74.01  & 98.10  & 94.03  & 96.02  & 95.33  & 97.30  & 96.30  \\
    TAP$_{\text{(ICLR'25)}}$ & 84.75 & 77.63  & 81.04  & 77.97  & 70.40 & 73.99  & \textbf{98.90}  & 95.50  & 97.17  & 95.80  & 97.73  & 96.76  \\
    CLIP-AST$_{\text{(CVPR'25)}}$ & 85.64  & 76.99 & 81.06  & 78.44 & 70.22  & 74.10  & 98.71  & 94.00  & 96.30  & 96.23  &  97.37 & 96.80  \\
    SurPL-G$_{\text{(ICML'25)}}$ & \textbf{86.37}  & 76.32 & 81.03  & \textbf{78.74} & 70.49  & 74.39  & 98.77  & 95.16  & 96.93  & 96.37  &  97.41 & 96.89  \\
    CoCoA-Mix$_{\text{(ICML'25)}}$ & 79.31  & 75.10 & 77.03  & 75.47 & 68.92 & 72.04  & 98.02 & 94.39  & 96.17  & 95.16  &  97.60 & 96.36  \\
    \midrule
    \rowcolor[HTML]{DEE9F9}
    UPrompt$_{\text{(Ours)}}$ & 86.35  & \textbf{78.29}  & \textbf{82.12} & 78.65  & \textbf{71.24}  & \textbf{74.76}  & 98.78  & \textbf{95.84}  & \textbf{97.29}  & \textbf{96.41}  & \textbf{97.92}  & \textbf{97.16}  \\
    \midrule
    \multirow{2}[4]{*}{Method} & \multicolumn{3}{c}{StanfordCars} & \multicolumn{3}{c}{Flowers102} & \multicolumn{3}{c}{Food101} & \multicolumn{3}{c}{FGVCAircraft} \\
\cmidrule{2-13}          & Base  & Novel & HM    & Base  & Novel & HM    & Base  & Novel & HM    & Base  & Novel & HM \\
    \midrule
    CoOp$_{\text{(IJCV'22)}}$ & 78.12  & 60.40  & 68.13  & 97.60  & 59.67  & 74.06  & 88.33  & 82.26  & 85.19  & 40.44  & 22.30  & 28.75  \\
    PSRC$_{\text{(ICCV'23)}}$ & 78.27  & 74.97  & 76.58  & 98.07  & 76.50  & 85.95  & 90.67  & 91.53  & 91.10  & 42.73  & 37.87  & 40.15  \\
    TAP$_{\text{(ICLR'25)}}$ &  80.70 & 74.27  & 77.35  & 97.90  & 75.37  & 85.30  & 90.97  & 91.83  & 91.40  & 40.40  & 36.50  & 40.06  \\
    CLIP-AST$_{\text{(CVPR'25)}}$ & \textbf{84.21}  & 74.05 & 78.80  & 97.91 & 77.73  & 86.66  & 90.57  & 91.11  & 90.84  & 48.98  &  38.21 & 42.93  \\
    SurPL-G$_{\text{(ICML'25)}}$ & 83.57  & 72.77 & 77.80  & \textbf{98.90} & 72.88  & 83.92  & 90.92  & 91.81  & 91.36  & 49.20  &  36.93 & 42.19  \\
    CoCoA-Mix$_{\text{(ICML'25)}}$ & 73.09  & \textbf{74.97} & 74.01  & 91.04 & 77.37 & 83.64  & 90.09 & 90.93  & 90.50  & 33.51  &  34.15 & 33.83  \\
    \midrule
    \rowcolor[HTML]{DEE9F9}
    UPrompt$_{\text{(Ours)}}$ & 83.58  & 74.57  & \textbf{78.82}  & 98.54  & \textbf{78.43}  & \textbf{87.34} & \textbf{91.20}  & \textbf{92.16}  & \textbf{91.68}  & \textbf{49.33}  & \textbf{39.25}  & \textbf{43.72}  \\
    \midrule
    \multirow{2}[4]{*}{Method} & \multicolumn{3}{c}{SUN397} & \multicolumn{3}{c}{DTD} & \multicolumn{3}{c}{EuroSAT} & \multicolumn{3}{c}{UCF101} \\
\cmidrule{2-13}          & Base  & Novel & HM    & Base  & Novel & HM    & Base  & Novel & HM    & Base  & Novel & HM \\
    \midrule
    CoOp$_{\text{(IJCV'22)}}$ & 80.60  & 65.89  & 72.51  & 79.44  & 41.18  & 54.24  & 93.19  & 54.74  & 68.69  & 84.69  & 56.05  & 67.46  \\
    PSRC$_{\text{(ICCV'23)}}$ & 82.67  & 78.47  & 80.52  & 83.37  & 62.97  & 71.75  & 92.90  & 73.90  & 82.32  & 87.10  & 78.80  & 82.74  \\
    TAP$_{\text{(ICLR'25)}}$ & 82.87  & 79.53  & 81.17  & 84.20  & \textbf{68.00}  & 75.24  & 90.70  & 82.17  & 86.22  & 87.90  & \textbf{82.43}  & 85.08  \\
    CLIP-AST$_{\text{(CVPR'25)}}$ & 83.05  & 78.12 & 80.51  & 84.03 & 65.34  & 73.52  & \textbf{95.90} & 81.72  & 88.24  & 87.38  &  79.12 & 83.05  \\
    SurPL-G$_{\text{(ICML'25)}}$ & 83.43  & 78.96 & 81.13  & \textbf{86.07} & 62.04  & 72.11  & 94.63  & 81.33  & 87.48  & \textbf{89.44}  & 79.74 & 84.31 \\
    CoCoA-Mix$_{\text{(ICML'25)}}$ & 78.51  & 76.60 & 77.54  & 72.80 & 64.29 & 68.25  & 83.49 & 69.11  & 75.54  & 81.28  &  77.75 & 79.47  \\
    \midrule
    \rowcolor[HTML]{DEE9F9}
    UPrompt$_{\text{(Ours)}}$ & \textbf{83.77}  & \textbf{80.05}  & \textbf{81.87}  & 85.60  & 67.23  & \textbf{75.31}  & 94.82  & \textbf{82.68}  & \textbf{88.33}  & 89.21  & 81.83  & \textbf{85.36}  \\
    \bottomrule
    \end{tabular}%
    }
    \vspace{-8pt}
  \label{tab_b2n}%
\end{table*}%

\begin{table}[t]
  \centering
  \caption{\textbf{Cross-modal retrieval performance} of different CLIP fine-tuning methods. ``ZS'' denotes zero-shot and ``FT'' is fine-tuned. rSum is the sum of all R@1, R@5, and R@10 scores. Best results highlighted in \colorbox[HTML]{d0e0f9}{\textbf{first}}, \colorbox[HTML]{e1ecfc}{second}.}
  \vspace{-7pt}
  \renewcommand{\arraystretch}{0.90}
  \setlength{\tabcolsep}{2.85pt}
  \scalebox{0.95}{
    \begin{tabular}{l|ccc|ccc|c}
    \toprule
    \multirow{3}[6]{*}{Methods} & \multicolumn{7}{c}{Flickr30K} \\
\cmidrule{2-8} 
    & \multicolumn{3}{c|}{Image-to-Text} & \multicolumn{3}{c|}{Text-to-Image} & \multirow{2}[4]{*}{rSum} \\
\cmidrule{2-7}
    & R@1   & R@5   & R@10  & R@1   & R@5   & R@10  &  \\
    \midrule
    CLIP$_\text{ZS}$ & 81.3  & 96.4  & 98.5  & 62.2  & 85.7  & 91.7  & 515.8 \\
    CLIP$_\text{FT}$ & 91.7  & 99.0  & 99.5  & 79.1  & 95.2  & 97.6  & 562.1 \\
    DoPL$_\text{(ACL'25)}$ & 69.8  & 90.7  & 95.0  & 66.9  & 89.0  & 93.6  & 505.0 \\
    MAMET$_\text{(TCSVT'25)}$ & 92.7  & \cellcolor[HTML]{e1ecfc}99.3  & \cellcolor[HTML]{e1ecfc}99.7  & 79.8  & 95.2  & 97.2  & 563.9 \\
    VPKE$_\text{(TCSVT'25)}$ & \cellcolor[HTML]{e1ecfc}93.7  & 99.2  & \cellcolor[HTML]{d0e0f9}\textbf{99.8}  & \cellcolor[HTML]{e1ecfc}82.0  & \cellcolor[HTML]{e1ecfc}95.7  & \cellcolor[HTML]{e1ecfc}98.2  & \cellcolor[HTML]{e1ecfc}568.6 \\
    \midrule
    UPrompt & \cellcolor[HTML]{d0e0f9}\textbf{93.8}  & \cellcolor[HTML]{d0e0f9}\textbf{99.4}  & 99.6  & \cellcolor[HTML]{d0e0f9}\textbf{83.6}  & \cellcolor[HTML]{d0e0f9}\textbf{96.3}  & \cellcolor[HTML]{d0e0f9}\textbf{98.4}  & \cellcolor[HTML]{d0e0f9}\textbf{571.1} \\
    \midrule\midrule
    \multirow{3}[6]{*}{Methods} & \multicolumn{7}{c}{MSCOCO} \\
\cmidrule{2-8}
    & \multicolumn{3}{c|}{Image-to-Text} & \multicolumn{3}{c|}{Text-to-Image} & \multirow{2}[4]{*}{rSum} \\
\cmidrule{2-7}
    & R@1   & R@5   & R@10  & R@1   & R@5   & R@10  &  \\
    \midrule
    CLIP$_\text{ZS}$ & 52.5  & 76.6  & 84.7  & 33.1  & 58.4  & 69.0  & 374.3 \\
    CLIP$_\text{FT}$ & 66.9  & 88.3  & 93.6  & 51.5  & 78.0  & 86.1  & 464.4 \\
    DoPL$_\text{(ACL'25)}$ & 63.2  & 86.7  & 91.8  & 49.6  & 76.3  & 85.2  & 452.8 \\
    MAMET$_\text{(TCSVT'25)}$ & 66.0  & 88.4  & 93.6  & 52.4  & \cellcolor[HTML]{d0e0f9}\textbf{79.3}  & \cellcolor[HTML]{e1ecfc}87.3  & 467.0 \\
    VPKE$_\text{(TCSVT'25)}$ & \cellcolor[HTML]{e1ecfc}69.2  & \cellcolor[HTML]{e1ecfc}89.0  & \cellcolor[HTML]{e1ecfc}94.2  & \cellcolor[HTML]{d0e0f9}\textbf{52.8}  & 78.5  & 86.5  & \cellcolor[HTML]{e1ecfc}470.2 \\
    \midrule
    UPrompt & \cellcolor[HTML]{d0e0f9}\textbf{70.1}  & \cellcolor[HTML]{d0e0f9}\textbf{89.8}  & \cellcolor[HTML]{d0e0f9}\textbf{94.8}  & \cellcolor[HTML]{e1ecfc}52.6  & \cellcolor[HTML]{e1ecfc}79.1  & \cellcolor[HTML]{d0e0f9}\textbf{87.9}  & \cellcolor[HTML]{d0e0f9}\textbf{474.3} \\
    \bottomrule
    \end{tabular}%
    }
  \label{tab:cross_modal_retrieval}%
  \vspace{-10pt}
\end{table}
\section{Experiments}
\label{experiments}
\noindent\textbf{Datasets.} For cross-modal retrieval, we evaluate on Flickr30K~\cite{young2014image} with 31,783 images and MSCOCO-5K~\cite{lin2014microsoft} with 123,287 images, each annotated with 5 captions. For classification tasks, we use 11 datasets: ImageNet~\cite{deng2009imagenet}, Caltech101~\cite{fei2004learning}, OxfordPets~\cite{parkhi2012cats}, StanfordCars~\cite{krause20133d}, Flowers102~\cite{nilsback2008automated}, Food101~\cite{bossard2014food}, FGVCAircraft~\cite{maji2013fine}, SUN397~\cite{xiao2016sun}, UCF101~\cite{soomro2012ucf101}, DTD~\cite{cimpoi2014describing} and EuroSAT~\cite{helber2019eurosat}. For out-of-distribution evaluation, we use ImageNet-A~\cite{hendrycks2021natural}, ImageNet-R~\cite{hendrycks2021many}, ImageNet-Sketch~\cite{wang2019learning}, and ImageNet-V2~\cite{recht2019imagenet}.

\noindent\textbf{Implementation details.} 
For multi-granularity construction, visual modality applies progressive spatial pooling to patch embeddings: from original 14×14 to 7×7, 4×4, and 1×1 tokens. Classification uses all 4 scales while retrieval uses the first 3 scales. For textual modality, we employ Llama 3-8B to generate hierarchical representations. In classification, we construct 4-level prompts:``\textit{a photo of a \{cls\}}'' (Level 1), progressively enriched with representative nouns (Level 2), attribute phrases (Level 3), and detailed descriptions (Level 4). In retrieval, original captions serve as medium granularity, with coarse granularity created via prompt``\textit{shorten the caption and keep important information}'', and fine granularity enhanced via ``\textit{add details from other captions to enhance the original caption and keep original meaning unchanged}''. Text hierarchies (via LLM) are pre-computed rather than computed in real-time, avoiding significant computational overhead.
We use CLIP ViT-B/16 as backbone by default (Ablation studies on different backbones and VLMs are in Appendix~\ref{appendix_different_backbones} and \ref{app_vlm_arch}). We employ 3- and 4-granularity configurations for retrieval and classification respectively, with distinct learnable prompts of length 4 for each granularity level. During inference, we default to the finest-grained features, which consistently yield optimal alignment. 

\subsection{Comparative Results} 
\noindent\textbf{Base-to-novel generalization.}
UPrompt achieves an 82.12\% harmonic mean (HM) across 11 datasets (Table~\ref{tab_b2n}), +1.06\% improvement over the second best method. It outperforms recent multi-level methods like TAP (81.04\%)~\cite{dingtree}, which constructs concept-attribute hierarchies, and SurPL-G (81.03\%)~\cite{liusurrogate}, which generates diverse features across granularities, as well as latest methods like CLIP-AST (81.06\%)~\cite{huang2025adaptive} and CoCoA-Mix (77.03\%)~\cite{hong2025cocoa}. While TAP and SurPL-G treat granularities as independent modules, UPrompt's U-shaped architecture establishes bidirectional information flow across hierarchical levels, providing more robust generalization with notable gains, including +0.79\% HM on the challenging FGVCAircraft. Appendix~\ref{app_error_bar} provides error bar analysis.

\begin{figure}[t]
  \centering
  \includegraphics[width=1\linewidth]{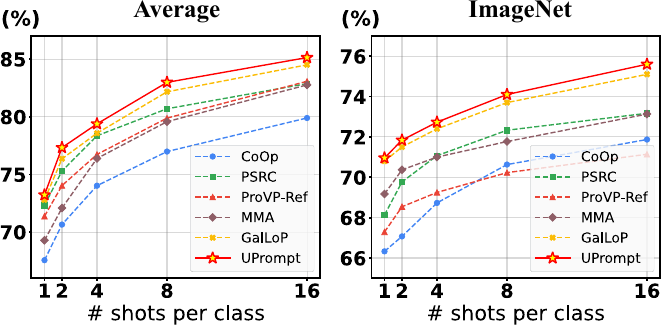}
  \vspace{-8pt}
  \caption{\textbf{Few-shot classification.} Performance on 11-dataset average and ImageNet. Remaining results are in the Appendix~\ref{other_few_shot}.}
  \label{few-shot_1}
  \vspace{-8pt}
\end{figure}

\begin{figure}[t]
\centering
\includegraphics[width=0.93\linewidth]{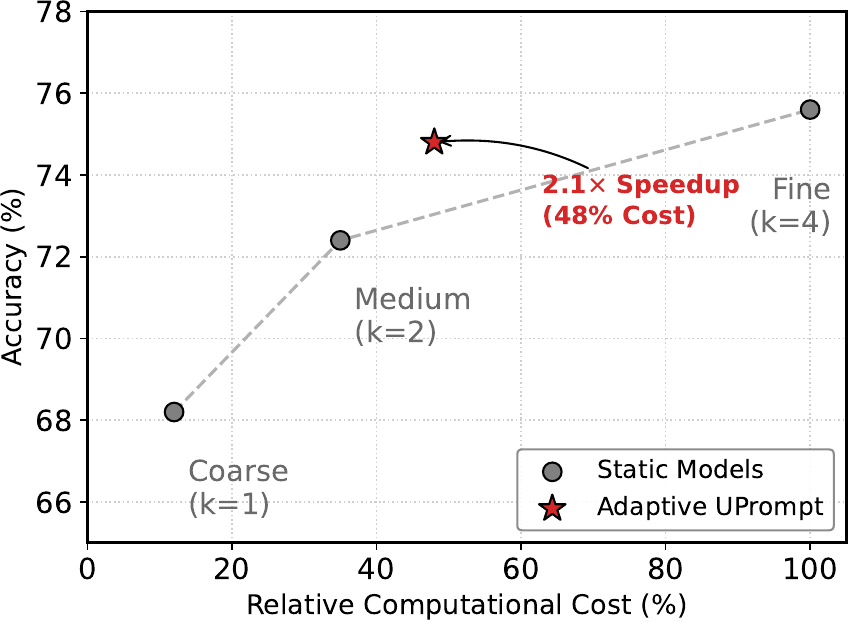}
    \vspace{-10pt}
    \caption{\textbf{Accuracy-Efficiency Trade-off on ImageNet.} Static models (grey circles) operate at fixed granularities. Adaptive UPrompt (red star) matches fine-grained accuracy while reducing computational cost to 48\%, achieving 2.1$\times$ speedup.}
    \label{fig:adaptive_tradeoff}
\vspace{-10pt}
\end{figure}

\noindent\textbf{Few-shot classification.} 
UPrompt achieves a leading 85.13\% averaged accuracy across 11 datasets at 16-shot (Fig.~\ref{few-shot_1}), surpassing recent methods: GalLoP~\cite{lafon2024gallop} (84.50\%), which learns from separate global and local features; ProVP-Ref~\cite{xu2025progressive} (83.07\%), which builds progressive layer-wise connections; and MMA~\cite{yang2024mma} (82.76\%), which adapts only higher-level representations. On ImageNet, our method performs best across all shots. Unlike these approaches lacking systematic cross-level interaction, UPrompt's bidirectional connection creates richer information flow across multi-granularity hierarchies, providing consistent gains from 1-shot to 16-shot.

\noindent\textbf{Adaptive Classification.}
We exploit the hierarchical architecture of UPrompt to enable adaptive inference via a cascaded early-exit strategy. Inference proceeds sequentially from the coarsest granularity ($k=1$) to the finest. At each level, the process terminates if the prediction confidence (maximum softmax probability) exceeds a calibrated threshold $\tau_k$; otherwise, it continues to the next level via cascaded enhancement.
To avoid data leakage, thresholds $\{\tau_k\}$ are determined via grid search on a held-out calibration set. For each threshold candidate, we compute the classification accuracy and average computational cost over the calibration set. We select thresholds that maximize early exits while maintaining accuracy comparable to the static fine-grained model. Efficiency is quantified using \textit{Average Relative Cost}:
$
    \text{Avg. Cost} = \frac{1}{N}\sum_{i=1}^{N} \frac{\mathcal{C}_{k_i}}{\mathcal{C}_K},
$
where $\mathcal{C}_{k_i}$ is the cumulative FLOPs required to reach exit level $k_i$, and $\mathcal{C}_K$ is the cost of the full model.

As illustrated in Figure~\ref{fig:adaptive_tradeoff}, Adaptive UPrompt achieves a favorable accuracy-efficiency trade-off. It attains accuracy comparable to the static fine-grained baseline (74.8\% vs. 75.6\%) while operating at only 48\% of the average relative cost, translating to approximately 2.1$\times$ speedup. Notably, we observe that most samples exit at coarse granularities ($k=1,2$), validating that our hierarchical design effectively routes easy samples to computationally cheaper levels while reserving fine-grained resources for challenging instances.

\begin{table}
  \centering
  \vspace{-5pt}
  \caption{\textbf{Out-of-distribution (OOD) generalization.} `*'~means reproduced results. Best results highlighted in \colorbox[HTML]{d0e0f9}{\textbf{first}}, \colorbox[HTML]{e1ecfc}{second}.}
  \vspace{-3pt}
  \setlength{\tabcolsep}{3.3pt}
  \renewcommand{\arraystretch}{0.98}
  \scalebox{0.93}{
    \begin{tabular}{lc|cccc|>{\columncolor[HTML]{f8f8f8}}c}
    \toprule
    \multirow{2}[4]{*}{\textbf{Method}} & \textbf{Source} & \multicolumn{5}{c}{\textbf{Target}} \\
\cmidrule{2-7}          & \textbf{ImgNet} & \textbf{-V2}  & \textbf{-S}    & \textbf{-A}    & \textbf{-R}    & \textbf{\textit{OOD}} \\
    \midrule
    CoOp & 71.51  & 64.44  & 47.61  & 49.53  & 74.98  & 59.14  \\
    PSRC  & 71.27  & 64.35  & 49.55  & 50.90  & 77.80  & \cellcolor[HTML]{e1ecfc}60.65  \\
    GalLoP* & 71.14  & 64.32  & \cellcolor[HTML]{e1ecfc}49.56  & 50.83  & 77.42  & 60.53  \\
    SPTR & 70.05  & 64.40  & 48.78  & \cellcolor[HTML]{e1ecfc}51.30  & \cellcolor[HTML]{e1ecfc}77.90  & 60.59  \\
    MMRL & \cellcolor[HTML]{e1ecfc}72.03  & \cellcolor[HTML]{e1ecfc}64.47  & 49.13  & 51.20  & 77.53  & 60.58  \\
    HiCroPL & 71.22  & 64.33  & 49.47  & 50.79  & 77.15  & 60.44  \\
    \midrule
    UPrompt  & \cellcolor[HTML]{d0e0f9}\textbf{72.25}  & \cellcolor[HTML]{d0e0f9}\textbf{65.06}  & \cellcolor[HTML]{d0e0f9}\textbf{50.43}  & \cellcolor[HTML]{d0e0f9}\textbf{51.33}  & \cellcolor[HTML]{d0e0f9}\textbf{78.04}  & \cellcolor[HTML]{d0e0f9}\textbf{61.22}  \\
    \bottomrule
    \end{tabular}%
   }
  \label{tab:ood}
    \vspace{-5pt}
\end{table}

\begin{table}[t]
  \centering
  \caption{\textbf{Ablation study on Flickr30K.} CE: coarse-to-fine \underline{C}ascaded \underline{E}nhancement; HS: fine-to-coarse \underline{H}ierarchical \underline{S}upervision. Baseline uses original image-text pairs; Fine-grained only extends baseline with finest-granularity features; ``gray'' is our default set.}
  \vspace{-5pt}
  \hspace{-10pt}
  \newcolumntype{C}[1]{>{\centering\arraybackslash}p{#1}}
  \scalebox{0.85}{
    \begin{tabular}{p{1.1cm}C{2.5cm}C{0.4cm}C{0.4cm}C{0.7cm}C{0.7cm}C{0.7cm}C{0.7cm}}
    \toprule
    \multirow{2}[4]{*}{Granularity} & \multirow{2}[4]{*}{Method} & \multicolumn{2}{c}{Strategy} & \multicolumn{2}{c}{I2T} & \multicolumn{2}{c}{T2I} \\
\cmidrule(lr){3-4} \cmidrule(lr){5-6} \cmidrule(lr){7-8}       &       & CE    & HS    & R@1   & R@5   & R@1   & R@5 \\
    \midrule
    \multirow{2}[2]{*}{Single} & Baseline & \ding{55}     & \ding{55}    & 91.2  & 98.1  & 80.0  & 94.6 \\
          & Fine-grained only & \ding{55}  & \ding{55}  & 92.2  & 98.3  & 81.1  & 95.2 \\
    \midrule
    \multirow{7}[6]{*}{Multiple} & \multicolumn{5}{l}{(\textit{Effectiveness of CE})}&       &  \\
          & Vision only & \ding{51}      &  \ding{55}   & 93.0  & 98.8  & 82.6  & 95.7 \\
          & Text only & \ding{51} & \ding{55}  & 92.7  & 98.6  & 82.4  & 95.5 \\
          & \cellcolor[HTML]{F2F2F2}Vision \& Text  & \cellcolor[HTML]{F2F2F2}\ding{51} & \cellcolor[HTML]{F2F2F2}\ding{55}  & \cellcolor[HTML]{F2F2F2}93.3  & \cellcolor[HTML]{F2F2F2}99.1  & \cellcolor[HTML]{F2F2F2}83.0  & \cellcolor[HTML]{F2F2F2}95.9 \\
\cmidrule{2-8}          & \multicolumn{5}{l}{(\textit{Effectiveness of HS})} &       &   \\
          & Coarse-to-fine & \ding{55} & \ding{51} & 86.8  & 95.1  & 75.3  & 91.5 \\
          & \cellcolor[HTML]{F2F2F2}Fine-to-coarse & \cellcolor[HTML]{F2F2F2}\ding{55}  & \cellcolor[HTML]{F2F2F2}\ding{51}  & \cellcolor[HTML]{F2F2F2}92.4  & \cellcolor[HTML]{F2F2F2}98.3  & \cellcolor[HTML]{F2F2F2}81.2  & \cellcolor[HTML]{F2F2F2}95.3 \\
\cmidrule{2-8}          & \cellcolor[HTML]{d2e2fa}Full & \cellcolor[HTML]{d2e2fa}\ding{51} & \cellcolor[HTML]{d2e2fa}\ding{51} & \cellcolor[HTML]{d2e2fa}\textbf{93.8}  & \cellcolor[HTML]{d2e2fa}\textbf{99.4}  & \cellcolor[HTML]{d2e2fa}\textbf{83.6}  & \cellcolor[HTML]{d2e2fa}\textbf{96.3} \\
    \bottomrule
    \end{tabular}%
  }
  \vspace{-10pt}
  \label{com_ablation}%
\end{table}

\begin{figure*}[htbp]
\centering
\begin{minipage}{0.35\textwidth}
  \centering
\includegraphics[width=1\textwidth]{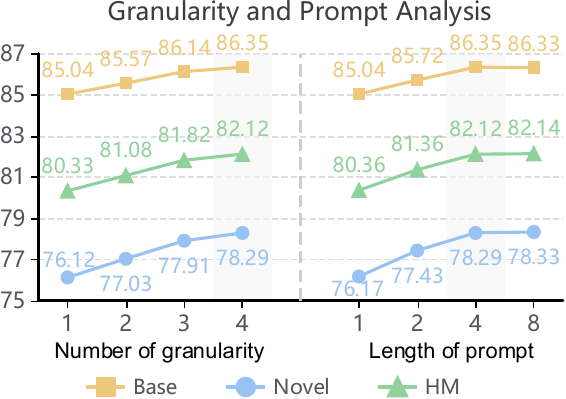}
      \vspace{-15pt}
      \caption{Granularity number and prompt length effects on classification. Left: varied granularity number (1-4). Right: varied prompt length (1-8).}
\label{grain_length_ablation}
\vspace{-5pt}
\end{minipage}
\hfill
\begin{minipage}{0.35\textwidth}
\vspace{3pt}
\centering
\includegraphics[width=1\textwidth]{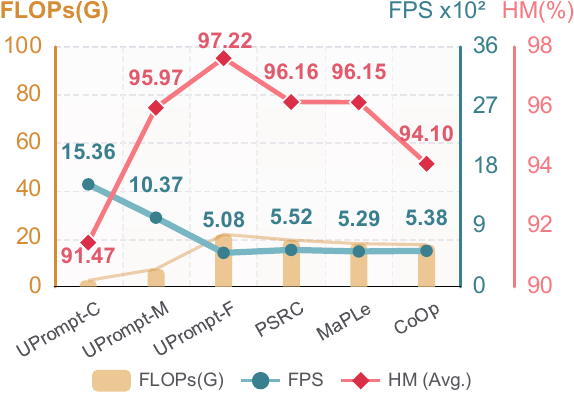}
      \vspace{-10pt}
      \caption{Efficiency-performance trade-offs across granularity levels. C, M, F denote coarse, medium, and fine configurations respectively.}
\label{fig_computational_cost}
\vspace{-5pt}
\end{minipage}
\hfill
\begin{minipage}{0.27\textwidth}
\vspace{5pt}
\centering
\captionof{table}{Text hierarchy robustness across different LLMs on Flickr30K.}
\vspace{-2pt}
\setlength{\tabcolsep}{4.2pt}
\renewcommand{\arraystretch}{0.95}
\scalebox{0.95}{
\begin{tabular}{lccc}
\toprule
\multicolumn{4}{c}{Image-to-text} \\
\midrule
Models & R@1 & R@5 & R@10 \\
\midrule
Qwen3-4B & 93.4 & 99.4 & 99.5 \\
\rowcolor[HTML]{DEE9F9}
Llama 3-8B & 93.8 & 99.4 & 99.6 \\
Qwen3-14B & 93.7 & 99.5 & 99.8 \\
\midrule\midrule
\multicolumn{4}{c}{Text-to-image} \\
\midrule
Models & R@1 & R@5 & R@10 \\
\midrule
Qwen3-4B & 83.4 & 96.2 & 98.6 \\
\rowcolor[HTML]{DEE9F9}
Llama 3-8B & 83.6 & 96.3 & 98.4 \\
Qwen3-14B & 83.9 & 96.2 & 98.7 \\
\bottomrule
\end{tabular}
}
\label{tab:diff_llm}
\end{minipage}
\end{figure*}

\noindent\textbf{Cross-modal retrieval.} We evaluate on Flickr30K and MSCOCO using Recall@K (K=1,5,10) and rSum (Table~\ref{tab:cross_modal_retrieval}). UPrompt achieves rSum scores of 571.1 and 474.3 on Flickr30K and MSCOCO, outperforming recent CLIP-based fine-tuning methods. It surpasses DoPL~\cite{guo2025parameter} (505.0, 452.8), which generates layer-wise prompts for alignment, MAMET~\cite{wang2025matryoshka} (563.9, 467.0), which distills knowledge from multiple embeddings, and VPKE~\cite{wang2025visualrag} (568.6, 470.2), which uses external visual knowledge. UPrompt's superior results, with R@1 scores of 93.8\% on Flickr30K and 70.1\% on MSCOCO, stem from its unique architecture. Its multi-granularity prompting captures vision-language alignments across various semantic levels, while hierarchical contextual guidance leads to significant improvement in retrieval precision. Appendix~\ref{app:adaptive_retrieval} further explores adaptive retrieval strategies.

\begin{table}
\large
\centering
\caption{\textbf{Resolution$\times$granularity level ablation} on Flickr30K (I2T R@1). Starting from the coarsest level, we progressively add finer granularity levels.}
\vspace{-3pt}
\setlength{\tabcolsep}{12pt}
  \renewcommand{\arraystretch}{1}
  \scalebox{0.98}{
\begin{tabular}{c|ccc}
\toprule
Resolution & Level 1 & Level 2 & Level 3 \\
\midrule
224×224 & 82.8 & 89.4 & 93.8 \\
336×336 & 83.3 & 92.2 & 95.1 \\
\bottomrule
\end{tabular}
}
\label{tab:resolution_ablation}
\vspace{-12pt}
\end{table}

\noindent\textbf{Out-of-distribution Generalization.} Domain shift evaluation examines semantic preservation (Table~\ref{tab:ood}). GalLoP~\cite{lafon2024gallop} sparse feature selection loses cross-domain information, SPTR~\cite{cui2025similarity} optimal transport maintains single-granularity stability, MMRL~\cite{guo2025mmrl} representation learning preserves generalization through decoupling, while HiCroPL~\cite{zheng2025hierarchical} bidirectional refinement focuses on task-specific alignment rather than domain robustness (60.44\%). UPrompt achieves 61.22\% through multi-granularity architecture with cascaded enhancement providing global guidance and hierarchical supervision preventing semantic drift.

\subsection{Ablation study}
\noindent\textbf{Component Ablation.}
We conduct an ablation study on Flickr30K (Table~\ref{com_ablation}) to isolate the contributions of our core components: coarse-to-fine cascaded enhancement (CE) and fine-to-coarse hierarchical supervision (HS). We use the ``Fine-grained only'' model (92.2\% I2T R@1) as our single-scale reference. 
Without CE, fine-grained features in multi-granularity architectures are very similar to single-granularity features. CE improves performance by maintaining global context while preserving fine details in ``Vision only'' mode (93.0\% I2T R@1), providing progressive textual enhancement in ``Text only'' mode (92.7\% I2T R@1), and enabling comprehensive local relationship modeling under global guidance when combined for ``Vision \& Text'' (93.3\% I2T R@1).
The directional nature of HS is critical; reversed coarse-to-fine supervision severely degrades performance by forcing detailed representations to model ambiguous signals. Conversely,  fine-to-coarse hierarchical supervision in isolation offers negligible improvement, as enforcing semantic consistency at coarse levels lacks a mechanism to refine independently optimized fine-grained prompts. The full UPrompt model achieves best results (93.8\% I2T R@1), revealing crucial connection where CE establishes cross-scale feature dependencies, allowing multi-level semantic consistency enforced by HS to effectively improve the entire representation hierarchy.

\begin{figure*}[!h]
    \centering    \includegraphics[width=0.93\linewidth]{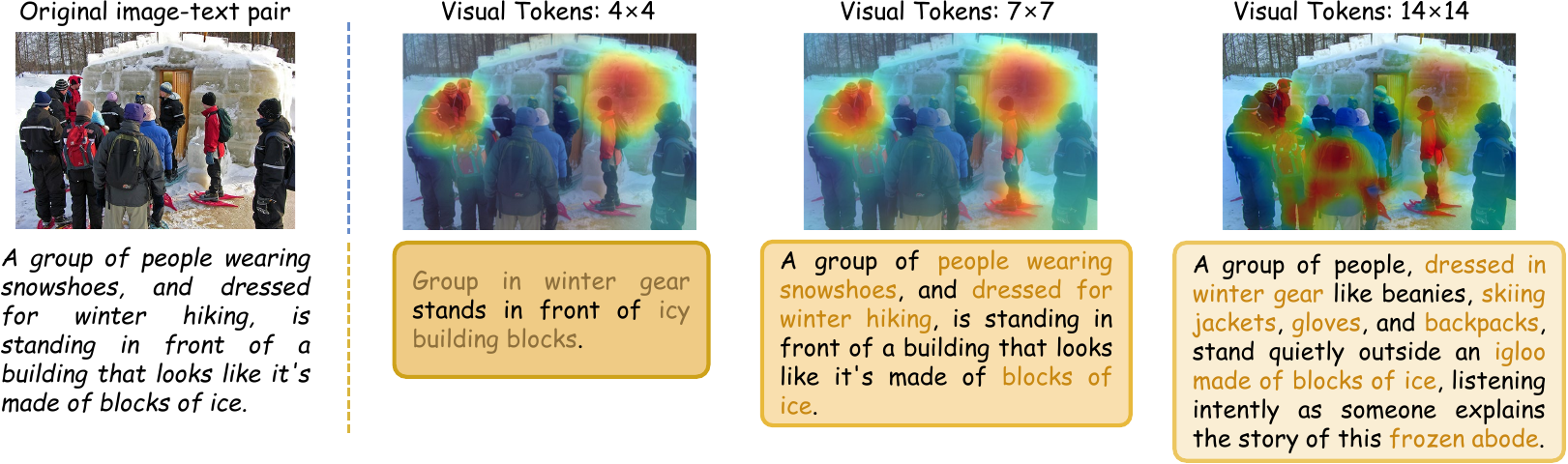}
    \vspace{-3pt}
    \caption{\textbf{Multi-granularity visual attention visualization.} Heat maps show bidirectional connection effects with cascaded enhancement preserving global context in fine-grained attention while hierarchical supervision reduces semantic drift across different levels.}
    \label{multi_grain_cam}
    \vspace{-8pt}
\end{figure*}
\noindent\textbf{Granularity Number and Prompt Length.}
We analyze sensitivity of UPrompt to the number of granularities and prompt lengths (Fig.~\ref{grain_length_ablation}). Starting from the finest level and progressively adding coarser levels, performance consistently improves as the number increases from 1 to 4, with harmonic mean reaching 82.12\% versus 80.33\% at single finest granularity, though improvement rate gradually slows with increasing computational overhead. We therefore adopt 4 granularity levels balancing performance and efficiency. For the downsampling strategy between granularities, following U-Net's 2$\times$2 downsampling principle~\cite{ronneberger2015u}, we employ approximate halving of visual token resolution ($14\times14\rightarrow7\times7\rightarrow4\times4$ for retrieval), which demonstrates superior performance over sparser intervals as analyzed in Appendix~\ref{app:interval_strategy}.
For prompt length, performance in Fig.~\ref{grain_length_ablation} (right) peaks at length 4 and remains stable at length 8, thus we set prompt length to 4.


\noindent\textbf{Resolution Analysis.}
To validate that multi-granularity improvements stem from semantic hierarchy rather than token density compensation, we conduct resolution-scale ablation on Flickr30K (Table~\ref{tab:resolution_ablation}). Results show that increasing the granularity level (Level 1→3) yields substantial gains at both 224×224 (82.8\%→93.8\%) and 336×336 (83.3\%→95.1\%), while resolution alone improves marginally (82.8\%→83.3\%). This demonstrates multi-granularity benefits dominate across resolutions, confirming our method addresses semantic hierarchy independent of spatial token density. Appendix~\ref{appendix_different_backbones} further validates consistent improvements with denser backbones (ViT-L), demonstrating framework generalization across model scales.

\noindent\textbf{Analysis of Performance-Cost.}
Fig.~\ref{fig_computational_cost} reveals the performance and cost across different granularities, where UPrompt-C, UPrompt-M, and UPrompt-F are coarse ($1\times1$), medium ($7\times7$), and fine ($14\times14$) visual tokens with corresponding textual granularities. UPrompt-F outperforms existing prompt learning methods with 97.22\% average HM across OxfordPets and Caltech101 with limited additional cost. UPrompt-M achieves comparable performance (95.97\% average HM) and matches PSRC's accuracy using only 1/3 of PSRC's FLOPs. UPrompt-C requires minimal resources while preserving reasonable performance at 91.47\% average HM. This flexibility stems from our architecture where fine-to-coarse supervision enables coarse levels to benefit from detailed representations, allowing adaptive granularity selection based on resource constraints.

\noindent\textbf{Robustness Analysis.}
We evaluate the sensitivity of our text hierarchy generation to different LLMs by conducting experiments with Qwen3-4B, Qwen3-14B~\cite{yang2025qwen3} and Llama3-8B on Flickr30K cross-modal retrieval (Table~\ref{tab:diff_llm}). The performance remains stable across models with varying architectures and scales (4B to 14B parameters), confirming our method is not sensitive to the specific LLM used. Appendix~\ref{rule_based_cons} validates that our bidirectional connection mechanism remains effective even with simple rule-based text hierarchies. Additional robustness validation across different backbones and VLM architectures is provided in Appendix~\ref{appendix_different_backbones} and~\ref{app_vlm_arch}.

\begin{table}[t]
  \centering
  \large
  \caption{\textbf{Effectiveness of Cascaded Enhancement (CE) and Hierarchical Supervision (HS)} on cross-modal retrieval on Flickr30K with different granularity sets.}
  \newcolumntype{C}[1]{>{\centering\arraybackslash}p{#1}}
  \scalebox{0.82}{
    \begin{tabular}{p{1.8cm}C{1.1cm}C{0.7cm}C{0.7cm}C{0.7cm}C{0.7cm}C{0.7cm}C{0.7cm}}
    \toprule
    \multirow{2}[4]{*}{Granularity} & \multirow{2}[4]{*}{Method} & \multicolumn{3}{c}{Image-to-text} & \multicolumn{3}{c}{Text-to-image} \\
\cmidrule(lr){3-5} \cmidrule(lr){6-8}          &       & R@1   & R@5   & R@10  & R@1   & R@5   & R@10 \\
    \midrule
    \multirow{2}[2]{*}{Coarse-grained} & w/o \textit{HS} & 75.4  & 89.9  & 92.8  & 60.4  & 82.3  & 86.8 \\
          & w/ \textit{HS} & \cellcolor[HTML]{dee9f9}82.8  & \cellcolor[HTML]{dee9f9}92.4  & \cellcolor[HTML]{dee9f9}93.5  & \cellcolor[HTML]{dee9f9}67.6  & \cellcolor[HTML]{dee9f9}87.5  & \cellcolor[HTML]{dee9f9}90.3 \\
    \midrule
    \multirow{2}[2]{*}{Fine-grained} & w/o \textit{CE} & 92.7  & 98.5  & 99.2  & 81.7  & 95.8  & 98.0  \\
          & w/ \textit{CE} & \cellcolor[HTML]{dee9f9}93.8  & \cellcolor[HTML]{dee9f9}99.4  & \cellcolor[HTML]{dee9f9}99.6  & \cellcolor[HTML]{dee9f9}83.6  & \cellcolor[HTML]{dee9f9}96.3  & \cellcolor[HTML]{dee9f9}98.4 \\
    \bottomrule
    \end{tabular}%
  }
  \label{tab:ce_and_hs_effect}%
  \vspace{-12pt}
\end{table}

\noindent\textbf{Analysis of Bidirectional Information Flow.}
We conduct an ablation study on Flickr30K to isolate our bidirectional connection components, with results in Table~\ref{tab:ce_and_hs_effect}. Fine-to-coarse Hierarchical Supervision (HS) substantially boosts coarse-grained performance (I2T R@1 from 75.4\% to 82.8\%), addressing the semantic drift caused by optimizing on the ambiguous signals inherent in simplified representations. Coarse-to-fine Cascaded Enhancement (CE) improves fine-grained performance (I2T R@1 from 92.7\% to 93.8\%), resolving context deficiency from isolated local detail modeling, without an understanding of their role within the global scene. Both components are integral: HS maintains semantic consistency for coarse representations, while CE provides contextual guidance for fine ones. To verify fine-grained supervision reliability, we compared single fine-layer supervision against mixed fine and medium-layer supervision (Appendix~\ref{Reliability of fine-grained supervision}). Results confirm fine-layer supervision alone achieves comparable performance, validating its sufficiency as the teacher signal.

\subsection{Visualization}
\noindent\textbf{Visualization of Multi-Granularity Attention.} 
Fig.~\ref{multi_grain_cam} validates UPrompt's bidirectional connection. Cascaded Enhancement enables fine-grained attention (14$\times$14) to maintain global coherence while capturing details like winter gear and textures. Hierarchical Supervision ensures coarser levels (4$\times$4, 7$\times$7) focus on semantically relevant regions, preventing background noise interference and semantic drift. The progressive refinement from global to local demonstrates effective multi-scale integration, where each granularity captures complementary information while maintaining semantic consistency. Fig.~\ref{coarse_also_yes_fig} in Appendix~\ref{appendix_granularity_specific} further shows our multi-granularity design across retrieval.

\noindent\textbf{Visualization of Bidirectional Connection Components.} 
Fig.~\ref{CAM_cascaded_enhancement_fine} and Fig.~\ref{CAM_hierarchical_supervision} in Appendix~\ref{appendix_vis_ce}, \ref{appendix_vis_hs} visualize Coarse-to-Fine Enhancement injects global context into fine-grained representations for improved local modeling, and Fine-to-Coarse Supervision leverages finest-level alignment to regularize and maintain consistency across coarser granularities.

\section{Conclusion}\label{conclusion}
In this work, we present UPrompt, a simple yet effective framework that addresses the limitation of single granularity in vision-language prompt learning. Inspired by U-Net, our U-shaped framework constructs parallel multi-granularity representations with bidirectional connections to facilitate information flow across scales. This consists of coarse-to-fine enhancement that injects global context into local details, and fine-to-coarse supervision that ensures semantic consistency. Extensive experiments demonstrate effectiveness across cross-modal retrieval, base-to-novel generalization, and few-shot classification. Despite its effectiveness, UPrompt's hierarchical depth remains limited, constraining its capacity to model richer semantic structures. Future work will explore deeper multi-granularity architectures with more hierarchical levels.

\bibliographystyle{ACM-Reference-Format}
\bibliography{sample-base}


\newpage
\appendix
\onecolumn

\clearpage
\appendix
\setcounter{theorem}{0}
{\Large \textbf{Appendix}}

The appendix provides supplementary materials including theoretical proofs, extended experiments, ablation studies, and visualizations. The contents are structured as follows:

\begin{itemize}
    \item Theoretical proofs for Propositions~\ref{proof_of_prop:ce} and~\ref{proof_of_prop:hs} (Appendix \hyperref[app_proof]{\textcolor{red}{A}}).

    \item Adaptive strategies for retrieval (Appendix \hyperref[app:adaptive_retrieval]{\textcolor{red}{B}}).
    
    \item Extended experiments on few-shot classification, cross-dataset transfer and error bar analysis (Appendix \hyperref[app_add_exp]{\textcolor{red}{C}}).
    
    \item Ablation studies on supervision reliability, granularity interval strategy and backbone architectures (Appendix \hyperref[app_abla]{\textcolor{red}{D}}).
    
    \item Sensitivity analysis on LLM-generated priors for text hierarchy construction (Appendix \hyperref[rule_based_cons]{\textcolor{red}{E}}).

    \item Visualizations of multi-granularity retrieval, cascaded enhancement, and hierarchical supervision (Appendix \hyperref[app_vis]{\textcolor{red}{F}}).
\end{itemize}

\section{Proof for the Proposition}
\label{app_proof}
\subsection{Proof for Proposition~\ref{prop:ce}}
\label{proof_of_prop:ce}
\begin{proposition}[CE Directional Alignment Effect]
Let $\hat{X}^{(k)}$ be the fine-grained representation at level $k$ enhanced by coarse-to-fine cascaded enhancement (CE, Eq.~(\ref{eq.5})–(\ref{eq.6})), which leverages contextual guidance from the coarser representation $\hat{X}^{(k-1)}$. Let $X^{(k)}$ be its unenhanced counterpart. Under the mild assumption that the coarse context is informative, CE provably strengthens the alignment between fine-grained features and their coarse-grained guidance in expectation:
\[
\mathbb{E}\!\left[ \frac{\langle \hat{X}^{(k)}, \hat{X}^{(k-1)} \rangle}{\|\hat{X}^{(k)}\| \,\|\hat{X}^{(k-1)}\|} \right] \;\ge\; \mathbb{E}\!\left[ \frac{\langle X^{(k)}, \hat{X}^{(k-1)} \rangle}{\|X^{(k)}\| \,\|\hat{X}^{(k-1)}\|} \right].
\]
\end{proposition}

\paragraph{Proof.}
Let $u \!\triangleq\! \hat{X}^{(k-1)}/\|\hat{X}^{(k-1)}\|$ be the unit coarse direction. Denote element-wise absolute value by $|\cdot|$, and define the “sign-adjusted” vector $v\!\in\!\mathbb{R}^d$ by $v_i \!\triangleq\! \mathrm{sign}(X^{(k)}_i u_i)\,|u_i|$. Then for any nonnegative gate $a\!\in\!\mathbb{R}^d_{\ge 0}$ we can rewrite the cosine with $u$ as
\[
\frac{\langle X^{(k)}\!\odot\! a,\ u\rangle}{\|X^{(k)}\!\odot\! a\|} \;=\; 
\frac{\big\langle |X^{(k)}|\!\odot\! a,\ v\big\rangle}{\big\|\,|X^{(k)}|\!\odot\! a\,\big\|}\,.
\]
Under CE (Eq.~(\ref{eq.5})–(\ref{eq.6})), the enhanced representation takes the form $\hat{X}^{(k)} \!=\! X^{(k)}\!\odot\! A\big(X^{(k)},\hat{X}^{(k-1)}\big)$ with an \emph{element-wise} nonnegative map $A$ produced from cross-granularity softmax attention and a value mixing of $\hat{X}^{(k-1)}$; we treat $a\!=\!A\big(X^{(k)},\hat{X}^{(k-1)}\big)$ as the random gate induced by CE.

\medskip\noindent
\textbf{Monotone informative-gate (MIG) assumption}~\cite{park2019relational}.
Formally, we instantiate the “coarse context is informative” condition as:
\begin{equation}
\label{eq:mig}
\mathbb{E}\!\big[a_i \,\big|\, X^{(k)},u\big] \ \text{is coordinatewise nondecreasing in}\ \ r_i \triangleq \frac{v_i}{|X^{(k)}_i|}\,,
\end{equation}
i.e., coordinates that are better aligned with $u$ (larger $r_i$) receive, in expectation, larger CE weights.\footnote{This matches the CE mechanism: the softmax attention in Eq.~(\ref{eq.6}) increases weights where the local embedding is more similar to the coarse context, and the resulting map gates $X^{(k)}$ element-wise in Eq. (\ref{eq.5}); see also the need to make such conditions explicit when CE is \emph{element-wise} rather than additive.}

\medskip\noindent
\textbf{Directional-derivative lemma.}
Define $f(a)\!\triangleq\!\dfrac{\langle |X^{(k)}|\!\odot\! a,\ v\rangle}{\big\|\,|X^{(k)}|\!\odot\! a\,\big\|}$. A direct calculation gives, for any coordinate $i$,
\[
\frac{\partial f}{\partial a_i}(a)
\;=\;
\frac{|X^{(k)}_i|\,v_i}{\big\|\,|X^{(k)}|\!\odot\! a\,\big\|}
\;-\;
f(a)\cdot \frac{a_i\,|X^{(k)}_i|^2}{\big\|\,|X^{(k)}|\!\odot\! a\,\big\|}\,.
\]
Evaluated at the identity gate $a\!=\!\mathbf{1}$, letting $f_0\!\triangleq\! f(\mathbf{1}) \!=\! \dfrac{\langle |X^{(k)}|,v\rangle}{\|X^{(k)}\|}$, we obtain the directional derivative along any perturbation $w$:
\[
\left.\frac{d}{d\alpha} f(\mathbf{1}+\alpha w)\right|_{\alpha=0}
\;=\;
\frac{1}{\|X^{(k)}\|}\sum_{i=1}^d w_i\,|X^{(k)}_i|\,\Big(\frac{v_i}{|X^{(k)}_i|}- f_0\Big)
\;=\;
\frac{1}{\|X^{(k)}\|}\sum_{i=1}^d w_i\,|X^{(k)}_i|\,\big(r_i- f_0\big).
\]
Hence the first-order increase of $f$ at $\mathbf{1}$ is nonnegative whenever $w$ is (on average) positively associated with the alignment score $r$.

\medskip\noindent
\textbf{Homotopy argument.}
Consider the linear path $a(t)\!=\!\mathbf{1}+t\,(a-\mathbf{1})$ for $t\!\in\![0,1]$. Differentiating along the path,
\[
\frac{d}{dt} f\big(a(t)\big)
\;=\;
\sum_{i=1}^d (a_i-1)\,\frac{\partial f}{\partial a_i}\big(a(t)\big).
\]
By continuity of $\frac{\partial f}{\partial a_i}$ and the preceding lemma, it suffices that the \emph{expected} increment $(a_i-1)$ remain positively associated with the local alignment score along the path. The MIG assumption (\ref{eq:mig}) guarantees exactly this: conditioned on $(X^{(k)},u)$, the coordinates with larger $r_i$ receive larger expected weights at every $t$, so $\mathbb{E}\!\left[\tfrac{d}{dt} f\big(a(t)\big)\,\middle|\, X^{(k)},u\right]\!\ge\!0$ for all $t\in[0,1]$. Integrating from $t\!=\!0$ to $t\!=\!1$ yields
\[
\mathbb{E}\!\big[f(a)\,\big|\, X^{(k)},u\big] \;\ge\; f(\mathbf{1}) \;=\; \frac{\langle X^{(k)},u\rangle}{\|X^{(k)}\|}.
\]
Finally, taking expectation over $(X^{(k)},\hat{X}^{(k-1)})$ proves
\[
\mathbb{E}\!\left[ \frac{\langle \hat{X}^{(k)}, \hat{X}^{(k-1)} \rangle}{\|\hat{X}^{(k)}\| \,\|\hat{X}^{(k-1)}\|} \right]
\;\ge\;
\mathbb{E}\!\left[ \frac{\langle X^{(k)}, \hat{X}^{(k-1)} \rangle}{\|X^{(k)}\| \,\|\hat{X}^{(k-1)}\|} \right].
\]
\hfill\qedsymbol

\medskip
\noindent\textbf{Remarks.}
(i) The proof explicitly uses CE’s \emph{element-wise} form (\ref{eq.5})–(\ref{eq.6}); additive/value-replacement assumptions are unnecessary. (ii) The MIG condition is a precise, verifiable sufficient condition tailored to element-wise gating, addressing the need to clarify when an element-wise CE improves directional alignment (and avoiding overly strong orthogonal-leakage assumptions).

\subsection{Proof of Proposition~\ref{prop:hs}}\label{proof_of_prop:hs}
\begin{proposition}[HS Consistency and Substitutability]\label{prop:hs}
Let $S^{(k)}$ and $S^{(K)}$ be similarity matrices from Eq.~(\ref{equation:similarity}), and define $p^{(k)}_{\tau_d}(j|i) = \text{softmax}(S^{(k)}_{i,:}/\tau_d)_j$ and $q^{(K)}_{\tau_d}(j|i) = \text{softmax}(S^{(K)}_{i,:}/\tau_d)_j$ where teacher $q^{(K)}$ is detached as in Eq.~(\ref{equation:hs}). Assuming HS aligns coarse-grained distributions with fine-grained teachers, HS bounds semantic drift and enables performance-preserving coarse inference:
\begin{equation}
\mathbb{E}_{(x,t),i}\!\left[\mathrm{KL}\!\left(q^{(K)}_{\tau_d}(\cdot|i) \,\|\, p^{(k)}_{\tau_d}(\cdot|i)\right)\right] \le \varepsilon \implies \mathbb{E}_{(x,t),i}\!\bigl[\bigl|\Phi\!\left(p^{(k)}_{\tau_d}(\cdot|i)\right)-\Phi\!\left(q^{(K)}_{\tau_d}(\cdot|i)\right)\bigr|\bigr] \le L\sqrt{\varepsilon/2}
\end{equation}
for any $L$-Lipschitz functional $\Phi$ w.r.t. total variation distance. The detach operation ensures gradient isolation: $\partial L_{guide}/\partial z^{(K)} = 0$.
\end{proposition}
\paragraph{Proof.}
Fix $(x,t)$ and anchor index $i$, and set 
$Q \coloneqq q^{(K)}_{\tau_d}(\cdot\mid i)$ and 
$P \coloneqq p^{(k)}_{\tau_d}(\cdot\mid i)$.
By Pinsker's inequality,
\[
\mathrm{TV}(P,Q)\;\le\;\sqrt{\tfrac12\,\mathrm{KL}(Q\,\|\,P)}.
\]
For any functional $\Phi$ that is $L$-Lipschitz w.r.t.\ total variation,
\[
\bigl|\Phi(P)-\Phi(Q)\bigr|
\;\le\; L\,\mathrm{TV}(P,Q)
\;\le\; L\sqrt{\tfrac12\,\mathrm{KL}(Q\,\|\,P)}.
\]
Taking expectation over $(x,t),i$ and applying Jensen's inequality (since $\sqrt{\cdot}$ is concave) yields
\[
\mathbb{E}_{(x,t),i}\!\bigl[\bigl|\Phi(P)-\Phi(Q)\bigr|\bigr]
\;\le\; L\,\mathbb{E}\!\left[\sqrt{\tfrac12\,\mathrm{KL}(Q\,\|\,P)}\right]
\;\le\; L\sqrt{\tfrac12\,\mathbb{E}\!\left[\mathrm{KL}(Q\,\|\,P)\right]}
\;\le\; L\sqrt{\varepsilon/2},
\]
which proves the stated consistency/substitutability bound.

For the gradient isolation, write the HS guidance loss as
\[
\mathcal{L}_{\text{guide}}
\;=\;
\mathbb{E}_{(x,t),i}\!\left[\mathrm{KL}\!\left(q^{(K)}_{\tau_d}(\cdot\mid i)\,\|\,p^{(k)}_{\tau_d}(\cdot\mid i)\right)\right],
\]
where the teacher $q^{(K)}_{\tau_d}$ is detached as in Eq.~(8).  
Hence $q^{(K)}_{\tau_d}$ is treated as a constant and 
\[
\frac{\partial \mathcal{L}_{\text{guide}}}{\partial z^{(K)}}=0.
\]
Equivalently, gradients flow only to the coarse head via the similarity logits $S^{(k)}$ from Eq.~(7):
if $P=\text{softmax}(S^{(k)}_{i,:}/\tau_d)$, then
\[
\frac{\partial \mathcal{L}_{\text{guide}}}{\partial S^{(k)}_{i,:}}
\;=\;\frac{1}{\tau_d}\bigl(P-Q\bigr),
\]
which is the standard soft-target distillation gradient scaled by $1/\tau_d$.
\hfill\qedsymbol

\medskip

\section{Adaptive Retrieval Strategy}
\label{app:adaptive_retrieval}

We implement a multi-stage cascaded pruning protocol to optimize the trade-off between retrieval precision and computational overhead. The inference pipeline operates progressively across granularity levels $k=1$ to $3$, dynamically reducing the gallery search space $\Omega_k$ based on prediction confidence. Initially, all $N$ gallery samples are evaluated at the coarse level, denoted as $|\Omega_1| = N$. At each stage $k$, we compute similarity scores for the current candidates and evaluate the retrieval certainty using the Relative Score Gap (RSG), defined as $\gamma_k = (s_{(1)} - s_{(1+m)}) / s_{(1)}$, where $s_{(r)}$ denotes the similarity score of the $r$-th ranked candidate and $m{=}10$ controls the rank gap for measuring confidence spread. If $\gamma_k$ exceeds a calibrated threshold $\tau_k$, the process terminates early, outputting the current top-1 prediction. Conversely, if $\gamma_k \leq \tau_k$ and $k < 3$, we retain the top-scoring half of $\Omega_k$ to construct $\Omega_{k+1}$, yielding $|\Omega_2|{=}N/2$ and $|\Omega_3|{=}N/4$ progressively.

To ensure generalization, the decision thresholds $\{\tau_1, \tau_2\}$ are calibrated via grid search on a held-out validation set. We select optimal threshold pairs that minimize average FLOPs subject to varying constraints on performance degradation. By adjusting this tolerance margin, we derive three representative operating points: \textit{Aggressive} (lower thresholds), \textit{Balanced}, and \textit{Conservative} (higher thresholds), corresponding to relaxed, moderate, and strict accuracy constraints, respectively. Specific configurations are detailed in Table~\ref{tab:adaptive_retrieval}. The results demonstrate that the \textit{Balanced} setting with $\tau_1{=}0.50$ 
and $\tau_2{=}0.65$ achieves 91.3\% R@1, retaining 97.3\% of the fine-grained 
performance (93.8\% from static Level 3), while reducing computational cost 
to 41.3\%.

\begin{table}[h]
\centering
\caption{Adaptive retrieval on Flickr30K image-to-text task. \textit{Exit Dist.} format: Coarse(\%)/Medium(\%)/Fine(\%). Thresholds $(\tau_1,\tau_2)$ denote RSG cutoffs at Levels 1 and 2 with margin $m=10$.}
\label{tab:adaptive_retrieval}
\begin{tabular}{lcccccc}
\toprule
Method & R@1 & R@5 & R@10 & Rel. Cost & Exit Dist. & Thresholds \\
\midrule
Static Level 1 & 82.8 & 92.4 & 93.5 & 14.2\% & 100/0/0 & \texttt{--} \\
Static Level 3 & 93.8 & 99.4 & 99.6 & 100.0\% & 0/0/100 & \texttt{--} \\
\midrule
Adaptive \textit{Aggressive} & 87.3 & 95.8 & 97.9 & 24.7\% & 76/18/6 & (0.30, 0.45) \\
Adaptive \textit{Balanced}   & 91.3 & 98.2 & 99.2 & 41.3\% & 52/28/20 & (0.50, 0.65) \\
Adaptive \textit{Conservative}& 93.1 & 98.9 & 99.4 & 67.1\% & 18/30/52 & (0.75, 0.85) \\
\bottomrule
\end{tabular}
\end{table}

\section{Additional Experiments}
\label{app_add_exp}
\subsection{Individual Dataset Few-shot Classification} 
\label{other_few_shot}
Results in Figure~\ref{few-shot_2} confirm UPrompt's effectiveness across diverse domains. On fine-grained tasks like StanfordCars and FGVCAircraft, our method outperforms CoOp and PSRC by capturing both specific details and global patterns through multi-granularity representations. For DTD texture classification, UPrompt surpasses MMA and ProVP-Ref via bidirectional connection that enables fine-grained modeling guided by coarse context. On SUN397 and EuroSAT, hierarchical supervision prevents semantic drift while maintaining competitive performance with GalLoP. Consistent improvements across shot configurations validate that our U-shaped architecture addresses granularity trade-offs in single-scale methods.
\begin{figure}[h]
    \centering    \includegraphics[width=1.0\linewidth]{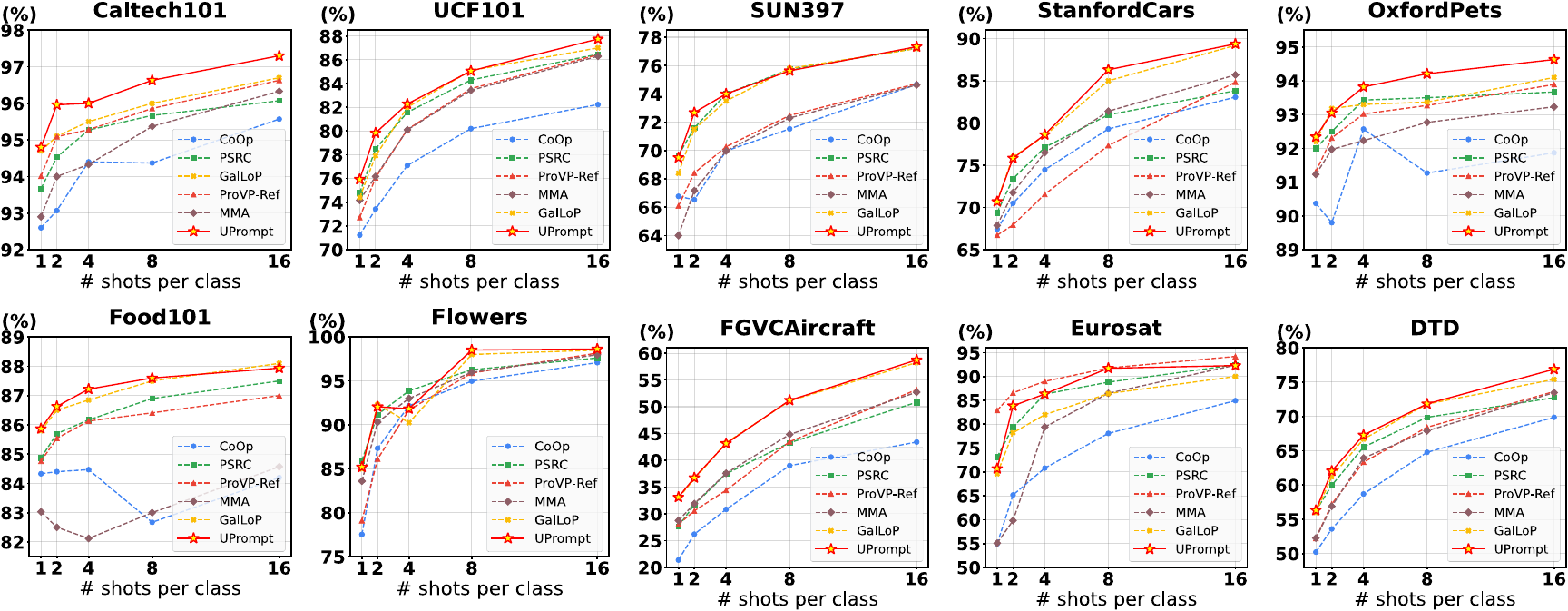}
    \vspace{-8pt}
    \caption{\textbf{Few-shot classification results on individual datasets.} Detailed performance breakdown across 10 evaluation datasets with 1, 2, 4, 8, and 16 shots per class.}
    \label{few-shot_2}
\end{figure}

\begin{table}[h]
  \centering
  \caption{\textbf{Cross-dataset evaluation.} Domain transfer against recent prompt learning methods. Trained on ImageNet, evaluated on 10 datasets. Best results highlighted in \colorbox[HTML]{d0e0f9}{\textbf{first}}, \colorbox[HTML]{e1ecfc}{second}.}
  \renewcommand{\arraystretch}{1}
  \scalebox{0.95}{
    \begin{tabular}{lc|cccccccccc>{\columncolor[HTML]{f8f8f8}}c}
    \toprule
    \multirow{2}[4]{*}{} & \textbf{Source} & \multicolumn{11}{c}{\textbf{Target}} \\
\cmidrule{2-13} & \rotatebox{60}{\textbf{ImgNet}} & \rotatebox{60}{\textbf{Cal101}} & \rotatebox{60}{\textbf{Pets}}   & \rotatebox{60}{\textbf{Cars}}  & \rotatebox{60}{\textbf{Flowers}} & \rotatebox{60}{\textbf{Food}}  & \rotatebox{60}{\textbf{FGVC}}  & \rotatebox{60}{\textbf{SUN}}  & \rotatebox{60}{\textbf{DTD}}  & \rotatebox{60}{\textbf{SAT}}  & \rotatebox{60}{\textbf{UCF}}  & \rotatebox{60}{\textbf{\textit{Avg.}}} \\
    \midrule
    \midrule
    CoOp$_{\text{(IJCV'22)}}$ & 71.51  & 93.70  & 89.14  & 64.51  & 68.71  & 85.30  & 18.47  & 64.15  & 41.92  & 46.39  & 66.55  & 63.88  \\
    PSRC$_{\text{(ICCV'23)}}$ & 71.27  & 93.60  & 90.25  & \cellcolor[HTML]{e1ecfc}65.70  & 70.25  & 86.15  & 23.90  & 67.10  & 46.87  & 45.50  & 68.75  & 65.81  \\
    DeKgTCP$_{\text{(ICLR'25)}}$ & \cellcolor[HTML]{e1ecfc}72.33  & \cellcolor[HTML]{e1ecfc}94.73  & 90.02  & 65.49  & \cellcolor[HTML]{e1ecfc}72.39  & \cellcolor[HTML]{e1ecfc}86.59  & 25.05  & 67.19  & 44.47  & \cellcolor[HTML]{e1ecfc}51.37  & 68.78  & 66.61  \\
    TAP$_{\text{(ICLR'25)}}$ & 72.30  & 94.30  & \cellcolor[HTML]{e1ecfc}90.70  & 65.60  & 70.93  & 86.10  & 24.57  & 68.30  & \cellcolor[HTML]{e1ecfc}50.20  & 46.00  & 68.90  & 66.56  \\
    TAC$_{\text{(CVPR'25)}}$ & \cellcolor[HTML]{d0e0f9}\textbf{72.77}  & 94.53  & 90.67  & 65.30  & 72.20  & 85.83  & 23.53  & 67.63  & 47.57  & 48.07  & 70.00  & 66.53  \\
    HiCroPL$_{\text{(ICCV'25)}}$ & 70.84  & 94.48  & 90.13  & 65.68  & 72.03  & 86.46  & \cellcolor[HTML]{d0e0f9}\textbf{26.58}  & \cellcolor[HTML]{d0e0f9}\textbf{68.78}  & \cellcolor[HTML]{d0e0f9}\textbf{53.19}  & 49.19  & \cellcolor[HTML]{d0e0f9}\textbf{70.31}  & \cellcolor[HTML]{e1ecfc}67.38  \\
    CoCoA-Mix$_{\text{(ICML'25)}}$ & 70.85  & 93.46  & 89.07  & 65.59  & 68.72  & 85.78  & 24.10  & 63.61  & 46.41  & 48.18  & 67.78  & 65.27  \\
    \midrule
    UPrompt$_{\text{(Ours)}}$ & 72.25  & \cellcolor[HTML]{d0e0f9}\textbf{94.75}  & \cellcolor[HTML]{d0e0f9}\textbf{90.97}  & \cellcolor[HTML]{d0e0f9}\textbf{66.09}  & \cellcolor[HTML]{d0e0f9}\textbf{72.41}  & \cellcolor[HTML]{d0e0f9}\textbf{86.60}  & \cellcolor[HTML]{e1ecfc}26.44  & \cellcolor[HTML]{e1ecfc}68.64  & 47.51  & \cellcolor[HTML]{d0e0f9}\textbf{52.08}  & \cellcolor[HTML]{e1ecfc}70.05  & \cellcolor[HTML]{d0e0f9}\textbf{67.55}  \\
    \bottomrule
    \end{tabular}%
}
  \label{tab:cross_dataset}%
\end{table}%

\subsection{Cross-dataset evaluation}
\label{appendix_cross_dataset}
As presented in Table~\ref{tab:cross_dataset}, our UPrompt framework demonstrates robust domain generalization capabilities when transferred from ImageNet to 10 downstream datasets. It achieves the highest average accuracy of 67.55\%, underscoring the effectiveness of its architecture in adapting to new data distributions. We compare UPrompt with other methods that also leverage multi-level or hierarchical representations. For instance, TAP~\cite{dingtree} constructs an explicit "concept-attribute-description" hierarchy, while HiCroPL~\cite{zheng2025hierarchical} establishes knowledge flow across network layers. Although these approaches are competitive, particularly HiCroPL on datasets like FGVC and SUN, our UPrompt's U-Net-inspired bidirectional multi-granularity learning leads to more consistent and superior performance across a wider range of tasks, securing the top results on 6 of the 10 target datasets. Furthermore, UPrompt outperforms other recent domain generalization methods like TAC~\cite{hao2025task} and DeKgTCP~\cite{li2025divergence}, validating that our explicit modeling of coarse-to-fine semantic levels is highly effective for robust cross-dataset transfer.


\begin{table}[!h]
  \centering
  \caption{\textbf{Error bar analysis on base-to-novel generalization.} Results report mean accuracy and standard deviation across three independent runs.}
  \setlength{\tabcolsep}{5pt}
  \renewcommand{\arraystretch}{1}
  \scalebox{0.95}{
    \begin{tabular}{l|ccc|ccc|ccc}
    \toprule
    \multirow{2}[2]{*}{\textbf{Method}} & \multicolumn{3}{c|}{\textbf{Average}} & \multicolumn{3}{c|}{\textbf{ImageNet}} & \multicolumn{3}{c}{\textbf{Caltech101}} \\
\cmidrule{2-10}
    & Base & Novel & HM & Base & Novel & HM & Base & Novel & HM \\
    \midrule
    CoOp & 82.69 & 63.22 & 71.66 & 76.47 & 67.88 & 71.92 & 96.00 & 89.81 & 93.73 \\
    \rowcolor[HTML]{DEE9F9}
    UPrompt & 86.35 & 78.29 & 82.12 & 78.65±0.09 & 71.24±0.11 & 74.76 & 98.78±0.04 & 95.84±0.13 & 97.29 \\
    \midrule
    \midrule
    \multirow{2}[2]{*}{\textbf{Method}} & \multicolumn{3}{c|}{\textbf{OxfordPets}} & \multicolumn{3}{c|}{\textbf{StanfordCars}} & \multicolumn{3}{c}{\textbf{Flowers102}} \\
\cmidrule{2-10}
    & Base & Novel & HM & Base & Novel & HM & Base & Novel & HM \\
    \midrule
    CoOp & 93.67 & 95.29 & 94.47 & 78.12 & 60.40 & 68.13 & 97.60 & 59.67 & 74.06 \\
    \rowcolor[HTML]{d2e2fa}
    UPrompt & 96.41±0.32 & 97.92±0.23 & 97.16 & 83.58±0.22 & 74.57±0.18 & 78.82 & 98.54±0.72 & 78.43±0.40 & 87.34 \\
    \midrule
    \midrule
    \multirow{2}[2]{*}{\textbf{Method}} & \multicolumn{3}{c|}{\textbf{Food101}} & \multicolumn{3}{c|}{\textbf{FGVCAircraft}} & \multicolumn{3}{c}{\textbf{SUN397}} \\
\cmidrule{2-10}
    & Base & Novel & HM & Base & Novel & HM & Base & Novel & HM \\
    \midrule
    CoOp & 88.33 & 82.26 & 85.19 & 40.44 & 22.30 & 28.75 & 80.60 & 65.89 & 72.51 \\
    \rowcolor[HTML]{d2e2fa}
    UPrompt & 91.20±0.17 & 92.16±0.22 & 91.68 & 49.33±0.30 & 39.25±0.16 & 43.72 & 83.77±0.19 & 80.05±0.32 & 81.87 \\
    \midrule
    \midrule
    \multirow{2}[2]{*}{\textbf{Method}} & \multicolumn{3}{c|}{\textbf{DTD}} & \multicolumn{3}{c|}{\textbf{EuroSAT}} & \multicolumn{3}{c}{\textbf{UCF101}} \\
\cmidrule{2-10}
    & Base & Novel & HM & Base & Novel & HM & Base & Novel & HM \\
    \midrule
    CoOp & 79.44 & 41.18 & 54.24 & 93.19 & 54.74 & 68.69 & 84.69 & 56.05 & 67.46 \\
    \rowcolor[HTML]{d2e2fa}
    UPrompt & 85.60±1.06 & 67.23±1.64 & 75.31 & 94.82±2.04 & 82.68±2.37 & 88.33 & 89.21±0.62 & 81.83±0.44 & 85.36 \\
    \bottomrule
    \end{tabular}%
    }
  \label{tab:error_bar_b2n}%
\end{table}
\begin{table}[!h]
  \centering
  \caption{\textbf{Error bar analysis on cross-dataset evaluation.} Trained on ImageNet, evaluated on 10 datasets. Results report mean accuracy and standard deviation across three independent runs.}
  \setlength{\tabcolsep}{3.2pt}
  \scalebox{1}{
    \begin{tabular}{c|ccccc}
    \toprule
    \textbf{Method} & \textbf{Caltech101} & \textbf{OxfordPets} & \textbf{StanfordCars} & \textbf{Flowers102} & \textbf{Food101} \\
    \midrule
    CoOp  & 93.70  & 89.14  & 64.51  & 68.71  & 85.30  \\
    \rowcolor[HTML]{d2e2fa}
    UPrompt & 94.51±0.11 & 90.75±0.17 & 65.98±0.08 & 72.19±0.23 & 86.57±0.03 \\
    \midrule
    \midrule
    \textbf{Method} & \textbf{FGVCAircraft} & \textbf{SUN397} & \textbf{DTD} & \textbf{EuroSAT} & \textbf{UCF101} \\
    \midrule
    CoOp  & 18.47 & 64.15 & 41.92 & 46.39 & 66.55 \\
    \rowcolor[HTML]{d2e2fa}
    UPrompt & 26.14±0.13 & 68.42±0.21 & 47.35±0.08 & 53.24±1.18 & 69.93±0.20 \\
    \bottomrule
    \end{tabular}%
    }
  \label{tab:error_bar}%
\end{table}
\subsection{Error bar analysis}
\label{app_error_bar}
We conducted error bar analysis across both cross-dataset and base-to-novel evaluation settings, performing three independent runs to ensure robust statistical evaluation. 
For base-to-novel generalization (Table~\ref{tab:error_bar_b2n}), UPrompt maintains consistent performance with low standard deviations across most datasets. While slightly higher variance appears on EuroSAT (±2.37) and DTD (±1.64) due to limited training samples, UPrompt still substantially outperforms CoOp, confirming the reliability of our bidirectional multi-granularity framework.
For cross-dataset evaluation (Table~\ref{tab:error_bar}), UPrompt demonstrates remarkable stability with low variance on Food101 (±0.03), StanfordCars (±0.08), and DTD (±0.08), validating robust cross-domain generalization.

\section{Other ablation studies}
\label{app_abla}

\subsection{Reliability of fine-grained supervision.}
\label{Reliability of fine-grained supervision}
To validate that fine-grained representations provide reliable supervision signals for coarser levels, we compared single fine-layer supervision with mixed supervision combining fine and medium-granularity teachers on Flickr30K and MSCOCO. Results in Table~\ref{tab:supervision_reliability} demonstrate that fine-layer supervision alone achieves comparable or superior performance across both datasets (Flickr30K I2T R@1: $93.8\%$ vs $93.7\%$; MSCOCO I2T R@1: $70.1\%$ vs $69.8\%$), confirming its sufficiency and stability as the primary teacher signal without requiring multi-layer aggregation.
\begin{table}[h]
  \centering
  \caption{\textbf{Fine-Layer Supervision Reliability.} Comparison of \textbf{fine-layer} versus \textbf{fine + medium-layer} supervision on Flickr30K and MSCOCO. rSum denotes sum of all R@1, R@5, R@10 scores.}
  \vspace{-3pt}
  \renewcommand{\arraystretch}{1}
  \setlength{\tabcolsep}{2pt}
  \scalebox{0.95}{
    \begin{tabular}{l|ccc|ccc|c||ccc|ccc|c}
    \toprule
    \multirow{3}[6]{*}{Teacher Strategy} & \multicolumn{7}{c||}{Flickr30K} & \multicolumn{7}{c}{MSCOCO} \\
\cmidrule{2-15} 
    & \multicolumn{3}{c|}{Image-to-Text} & \multicolumn{3}{c|}{Text-to-Image} & \multirow{2}[4]{*}{rSum} & \multicolumn{3}{c|}{Image-to-Text} & \multicolumn{3}{c|}{Text-to-Image} & \multirow{2}[4]{*}{rSum} \\
\cmidrule{2-7}\cmidrule{9-14}
    & R@1   & R@5   & R@10  & R@1   & R@5   & R@10  &  & R@1   & R@5   & R@10  & R@1   & R@5   & R@10  &  \\
    \midrule
    Fine + Medium (soft mixing) & 93.7 & 99.2  & 99.3  & 83.9  & 96.3  & 98.2  & 570.6 & 69.8  & 89.4  & 94.5  & 52.5  & 78.8  & 87.5  & 472.5 \\
    Fine only (Ours) & 93.8  & 99.4  & 99.6  & 83.6  & 96.3  & 98.4  & 571.1 & 70.1  & 89.8  & 84.8  & 52.6  &  79.1  & 87.9  & 474.3 \\
    \bottomrule
    \end{tabular}%
    }
  \label{tab:supervision_reliability}%
  \vspace{-5pt}
\end{table}

\begin{wraptable}{r}{0.45\textwidth}
\centering
\vspace{-12pt}
\caption{Granularity interval strategy comparison on Flickr30K I2T retrieval using R@1.}
\vspace{-4pt}
\setlength{\tabcolsep}{5pt}
  \renewcommand{\arraystretch}{1}
  \scalebox{1}{
\begin{tabular}{l|ccc}
\toprule
Strategy & Coarse & Medium & Fine \\
\midrule
$14\times14\rightarrow5\times5\rightarrow2\times2$ & 78.6 & 88.3 & 93.6 \\
\rowcolor[HTML]{DEE9F9}
$14\times14\rightarrow7\times7\rightarrow4\times4$ & 82.8 & 89.4 & 93.8 \\
\bottomrule
\end{tabular}
}
\label{tab:interval_strategy}
\vspace{-5pt}
\end{wraptable}

\subsection{Granularity Interval Strategy}
\label{app:interval_strategy}
Beyond the number of granularity levels, we examine different downsampling interval strategies between adjacent granularities. Table~\ref{tab:interval_strategy} compares two strategies on Flickr30K I2T retrieval: our adopted approach following U-Net's principle~\cite{ronneberger2015u} of approximate halving ($14\times14\rightarrow7\times7\rightarrow4\times4$) versus a sparser interval ($14\times14\rightarrow5\times5\rightarrow2\times2$). Both strategies achieve comparable performance at the fine-grained level (93.6\% vs 93.8\%), but the sparse interval shows degradation at coarser levels, particularly at the coarse granularity (78.6\% vs 82.8\%). This performance gap is attributed to the reduced expressive capacity at coarser levels under sparser downsampling, demonstrating the effectiveness of our graduated interval design.

\begin{table}[h]
  \centering
  \caption{\textbf{Cross-modal retrieval results on ViT-B/32 backbone.} rSum is the sum of all R@1, R@5, and R@10 scores. Best results highlighted in \colorbox[HTML]{d0e0f9}{\textbf{first}}, \colorbox[HTML]{e1ecfc}{second}.}
  \setlength{\tabcolsep}{3pt}
  \scalebox{0.95}{
    \begin{tabular}{l|ccc|ccc|c||ccc|ccc|c}
    \toprule
    \multirow{3}[6]{*}{Methods} & \multicolumn{7}{c||}{Flickr30K} & \multicolumn{7}{c}{MSCOCO} \\
\cmidrule{2-15} 
    & \multicolumn{3}{c|}{Image-to-Text} & \multicolumn{3}{c|}{Text-to-Image} & \multirow{2}[4]{*}{rSum} & \multicolumn{3}{c|}{Image-to-Text} & \multicolumn{3}{c|}{Text-to-Image} & \multirow{2}[4]{*}{rSum} \\
\cmidrule{2-7}\cmidrule{9-14}
    & R@1   & R@5   & R@10  & R@1   & R@5   & R@10  &  & R@1   & R@5   & R@10  & R@1   & R@5   & R@10  &  \\
    \midrule
    MAMET$_{\text{(TCSVT'25)}}$ & \cellcolor[HTML]{e1ecfc}87.7  & 97.5  & \cellcolor[HTML]{e1ecfc}99.6  & \cellcolor[HTML]{e1ecfc}73.5  & \cellcolor[HTML]{e1ecfc}93.0  & \cellcolor[HTML]{e1ecfc}96.5  & \cellcolor[HTML]{e1ecfc}547.8 & \cellcolor[HTML]{e1ecfc}61.5  & \cellcolor[HTML]{e1ecfc}86.2  & 92.5  & \cellcolor[HTML]{e1ecfc}48.6  & \cellcolor[HTML]{e1ecfc}76.3  & \cellcolor[HTML]{e1ecfc}85.3  & \cellcolor[HTML]{e1ecfc}450.4 \\
    APSE-IPIK$_{\text{(AAAI'25)}}$ & 86.3  & \cellcolor[HTML]{e1ecfc}97.6  & 99.4  & 72.0  & 92.5  & 95.1  & 542.9 & 59.1  & 85.7  & \textbf{94.6}  & 45.1  & 72.8  & 82.5  & 439.8 \\
    \midrule
    UPrompt$_{\text{(Ours)}}$ & \cellcolor[HTML]{d0e0f9}\textbf{88.9}  & \cellcolor[HTML]{d0e0f9}\textbf{97.6}  & \cellcolor[HTML]{d0e0f9}\textbf{99.7}  & \cellcolor[HTML]{d0e0f9}\textbf{74.0}  & \cellcolor[HTML]{d0e0f9}\textbf{93.2}  & \cellcolor[HTML]{d0e0f9}\textbf{96.8}  & \cellcolor[HTML]{d0e0f9}\textbf{550.2} & \cellcolor[HTML]{d0e0f9}\textbf{62.4}  & \cellcolor[HTML]{d0e0f9}\textbf{86.8}  & \cellcolor[HTML]{e1ecfc}93.8  & \cellcolor[HTML]{d0e0f9}\textbf{49.7}  & \cellcolor[HTML]{d0e0f9}\textbf{76.9}  & \cellcolor[HTML]{d0e0f9}\textbf{85.7}  & \cellcolor[HTML]{d0e0f9}\textbf{455.3} \\
    \bottomrule
    \end{tabular}%
    }
  \label{tab:vit_b32_retrieval_results}%
\end{table}

\begin{table}
  \centering
  \caption{\textbf{Image-to-text retrieval on ViT-L/14 backbone.} Results highlighted in \colorbox[HTML]{d0e0f9}{\textbf{first}}, \colorbox[HTML]{e1ecfc}{second}.}
  \vspace{-2pt}
  \renewcommand{\arraystretch}{1}
  \scalebox{1}{
    \begin{tabular}{l|ccc|ccc|c}
    \toprule
    \multirow{2}[4]{*}{Methods} & \multicolumn{3}{c|}{MSCOCO} & \multicolumn{4}{c}{Flickr30K} \multirow{2}[4]{*}{rSum} \\
\cmidrule{2-4}\cmidrule{5-7}
    & R@1   & R@5   & R@10  & R@1   & R@5   & R@10  &  \\
    \midrule
    Unicoder-VL & 62.3  & 87.1  & 92.8  & 86.2  & 96.3  & 99.0  & 523.7 \\
    Oscar & 73.5  & 92.2  & 96.0  & -  & -  & -  & - \\
    ERNIE-ViL & - & - & - & 88.7  & 98.0  & 99.2  & - \\
    AAPE & \cellcolor[HTML]{e1ecfc}76.7  & \cellcolor[HTML]{e1ecfc}94.5  & \cellcolor[HTML]{e1ecfc}97.4  & \cellcolor[HTML]{e1ecfc}94.9  & \cellcolor[HTML]{e1ecfc}99.3  & \cellcolor[HTML]{e1ecfc}99.7  & \cellcolor[HTML]{e1ecfc}561.8 \\
    \midrule
    UPrompt& \cellcolor[HTML]{d0e0f9}\textbf{77.8}  & \cellcolor[HTML]{d0e0f9}\textbf{94.9}  & \cellcolor[HTML]{d0e0f9}\textbf{97.4}  & \cellcolor[HTML]{d0e0f9}\textbf{95.1}  & \cellcolor[HTML]{d0e0f9}\textbf{99.7}  & \cellcolor[HTML]{d0e0f9}\textbf{99.8}  & \cellcolor[HTML]{d0e0f9}\textbf{564.7} \\
    \bottomrule
    \end{tabular}%
    }
  \label{tab:vit_l14_image_to_text}%
  \vspace{-5pt}
\end{table}
\subsection{Cross-modal retrieval with different backbones}
\label{appendix_different_backbones}
To validate UPrompt's generalizability across different architectures, we evaluate on ViT-B/32 and ViT-L/14 backbones (Tables~\ref{tab:vit_b32_retrieval_results} and \ref{tab:vit_l14_image_to_text}). On ViT-B/32, UPrompt achieves 550.2 and 455.3 rSum on Flickr30K and MSCOCO respectively, outperforming MAMET~\cite{wang2025matryoshka} (547.8, 450.4) and APSE-IP1K~\cite{huang2025adaptive} (542.9, 439.8). The consistent improvements across different model scales demonstrate the robustness of our multi-granularity framework. On the larger ViT-L/14 backbone for image-to-text retrieval, UPrompt achieves 564.7 rSum, outperforming AAPE~\cite{huang2024aggregate} (561.8), Unicoder-VL~\cite{li2020unicoder} (523.7), Oscar~\cite{li2020oscar}, and ERNIE-ViL~\cite{yu2021ernie}, with particularly strong performance on Flickr30K (95.1\% R@1) and MSCOCO (77.8\% R@1). These results confirm that our bidirectional connection mechanisms effectively leverage increased model capacity, with hierarchical supervision preventing semantic drift across granularities regardless of backbone architecture.

\begin{figure}[htbp]
\centering
\begin{minipage}{0.48\textwidth}
  \centering
  \captionof{table}{\textbf{Base-to-novel generalization on alternative VLM architectures.} UPrompt consistently outperforms baselines across SigLIP and EVA-CLIP.}
  \vspace{1pt}
  \renewcommand{\arraystretch}{1.2}
  \scalebox{1}{
    \begin{tabular}{l|l|ccc}
    \toprule
    Methods & Backbone & Cars & Flowers & FGVC \\
    \midrule
    CoOp  & EVA-CLIP & 71.33 & 77.39 & 34.72 \\
    \rowcolor[HTML]{d2e2fa}
    UPrompt & EVA-CLIP & \textbf{79.42} & \textbf{87.36} & \textbf{43.86} \\
    CoOp  & SigLIP & 92.33 & 89.42 & 38.27 \\
    \rowcolor[HTML]{d2e2fa}
    UPrompt & SigLIP & \textbf{94.67} & \textbf{93.26} & \textbf{46.34} \\
    \bottomrule
    \end{tabular}%
   }
  \label{tab:alternative_VLM}%
\end{minipage}
\hfill
\begin{minipage}{0.48\textwidth}
  \centering
  \captionof{table}{\textbf{Rule-based hierarchies} on CUB-200 and AWA2 in generalized zero-shot learning.}
  \vspace{-4pt}
  \setlength{\tabcolsep}{8pt}
  \renewcommand{\arraystretch}{0.79}
  \scalebox{0.86}{
    \begin{tabular}{llcccc}
    \toprule
    Dataset & Method & Level & Base  & New   & HM \\
    \midrule
    \multirow{4}[2]{*}{AWA2} & CoOp  & Single & 95.32 & 72.68 & 82.47 \\
    \cmidrule{2-6}
          & \multirow{3}[1]{*}{Uprompt} & Level 1 & 93.24 & 70.27 & 80.14 \\
          &       & Level 2 & 95.81 & 73.13 & 82.95 \\
          &       & Level 3 & \textbf{96.70} & \textbf{74.62} & \textbf{84.24} \\
    \midrule
    \midrule
    \multirow{4}[2]{*}{CUB-200} & CoOp  & Single & 63.78 & 49.23 & 55.57 \\
    \cmidrule{2-6}
          & \multirow{3}[1]{*}{UPrompt} & Level 1 & 61.36 & 46.84 & 53.13 \\
          &       & Level 2 & 64.65 & 50.42 & 56.65 \\
          &       & Level 3 & \textbf{66.12} & \textbf{51.26} & \textbf{57.75} \\
    \bottomrule
    \end{tabular}%
    }
  \label{tab:rule_based text construction}
\end{minipage}
\vspace{-5pt}
\end{figure}

\subsection{Generalization across VLM architecture}
\label{app_vlm_arch}
To validate UPrompt's generalizability across diverse vision-language models, we evaluate on alternative architectures including SigLIP and EVA-CLIP with ViT-B/16 backbones. Table~\ref{tab:alternative_VLM} presents base-to-novel generalization results on three representative datasets. UPrompt achieves consistent improvements over CoOp across both architectures: on EVA-CLIP, gains of +8.09\%, +9.97\%, and +9.14\% HM on StanfordCars, Flowers102, and FGVCAircraft respectively; on SigLIP, gains of +2.34\%, +3.84\%, and +8.07\% HM respectively. Notably, the improvements are particularly pronounced on the fine-grained FGVCAircraft dataset (+9.14\% on EVA-CLIP, +8.07\% on SigLIP), demonstrating that our multi-granularity framework effectively enhances fine-grained recognition across different VLM architectures. These results confirm that our bidirectional connection mechanisms operate effectively at the prompt level, independent of the underlying vision-language model architecture.

\section{Rule-Based Text Construction.}
\label{rule_based_cons}
To demonstrate that our framework's effectiveness does not rely on LLM-generated priors, we conduct experiments using rule-based text construction on CUB-200~\cite{wah2011caltech} and AWA2~\cite{xian2017zero} datasets, which provide dense attribute annotations. We construct three granularity levels: Level 1 (coarse) uses "a photo of a {class}", Level 2 (medium) adds the highest-certainty attribute (e.g., "a photo of a {class} with black bill"), and Level 3 (fine) incorporates two highest-certainty attributes (e.g., "a photo of a {class} with black bill and white breast"), where attributes are ranked by certainty scores from dataset metadata. Visually, the three levels correspond to 4×4 pooled tokens, 7×7 pooled tokens, and the original 14×14 tokens respectively. We compare against CoOp, which uses "a photo of a {class}" with 14×14 visual tokens as the single-granularity baseline. Table~\ref{tab:rule_based text construction} presents results on generalized zero-shot learning (GZSL). The progressive improvements from coarse to fine levels validate that our bidirectional connection mechanism effectively integrates multi-scale representations, even without external knowledge sources like LLM generated content.

\begin{figure}[!h]
    \centering    \includegraphics[width=1.0\linewidth]{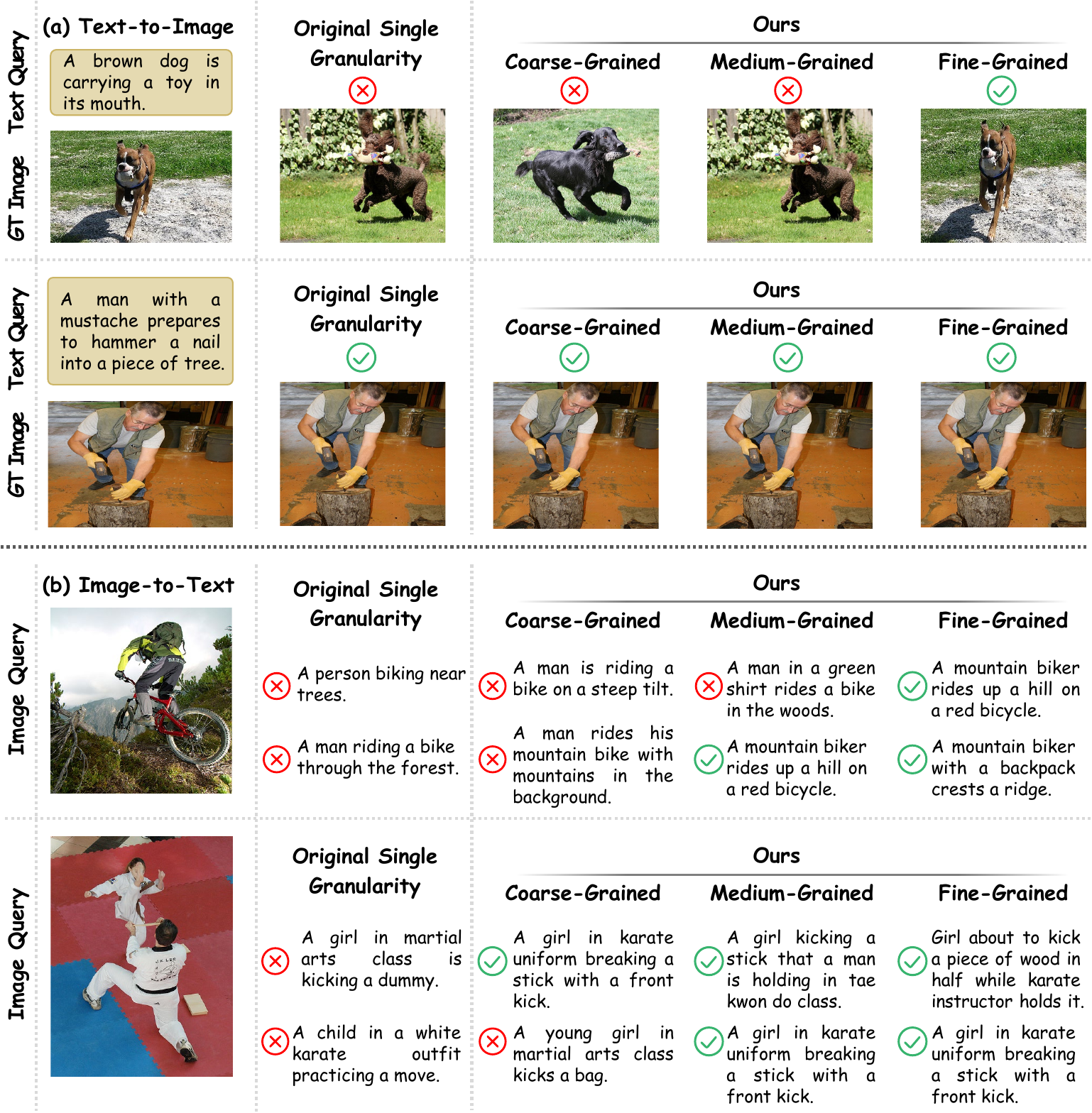}
    \vspace{-8pt}
    \caption{\textbf{Cross-modal retrieval results on Flickr30K dataset.} ``Original Single Granularity'' refers to baseline model using fixed single-scale visual and textual representations. \textcolor[HTML]{FF0000}{\textcircled{$\times$}} indicate retrieval failures, \textcolor[HTML]{37B46E}{\textcircled{$\checkmark$}} indicate successful retrievals.}
    \label{coarse_also_yes_fig}
\end{figure}

\begin{figure}[!h]
    \centering    \includegraphics[width=1.0\linewidth]{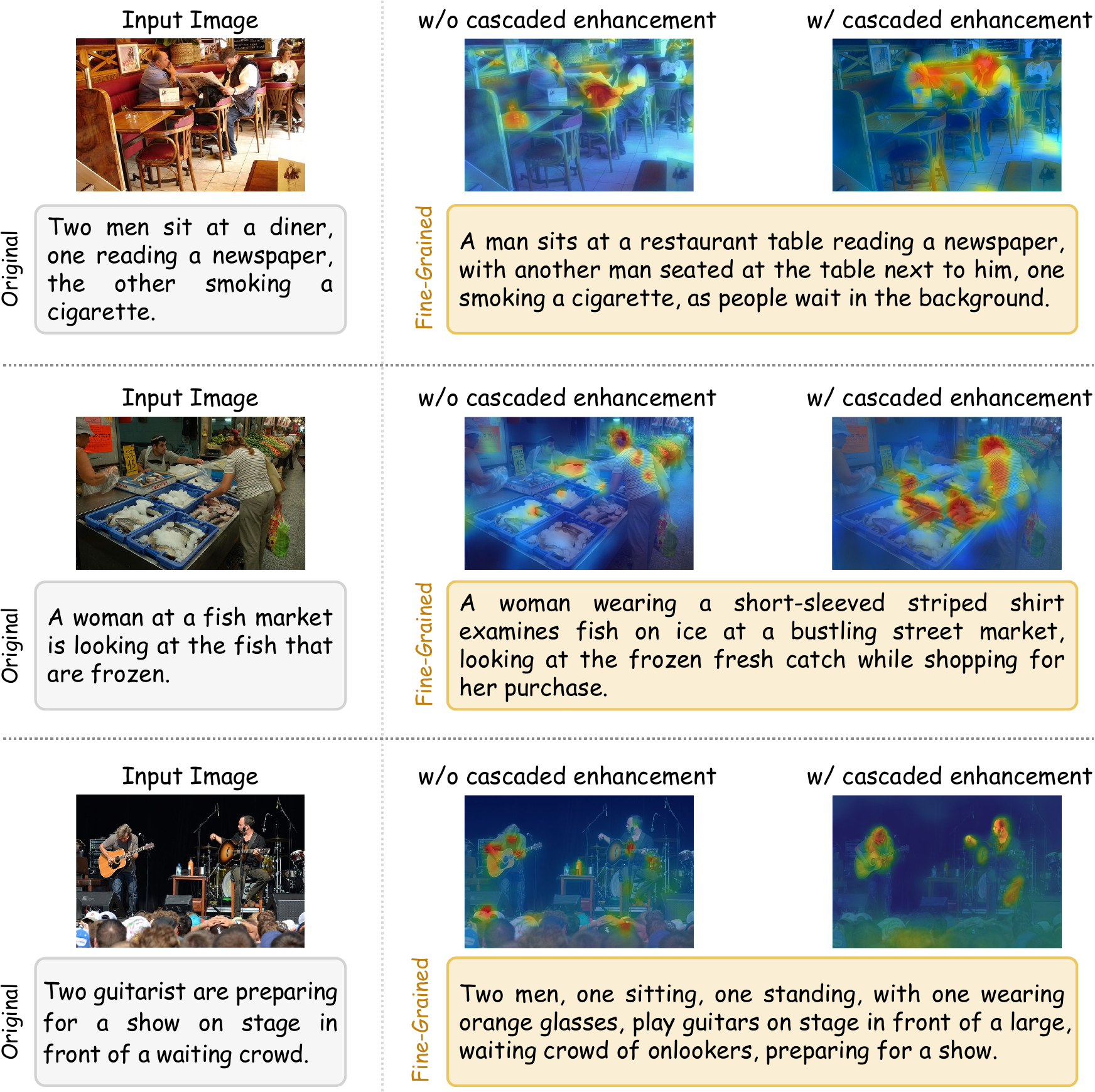}
    \vspace{-8pt}
    \caption{\textbf{Visual validation of Coarse-to-Fine Cascaded Enhancement (CE).} Our CE module addresses the context deficiency of fine-grained embeddings by injecting global contextual guidance. The comparison demonstrates that without CE (middle column), fine-grained attention struggles to model local information relationships. With CE (right column), our model achieves precise, contextually-aware alignment for complex fine-grained descriptions.}
    \label{CAM_cascaded_enhancement_fine}
\end{figure}

\begin{figure}[!h]
    \centering    \includegraphics[width=1.0\linewidth]{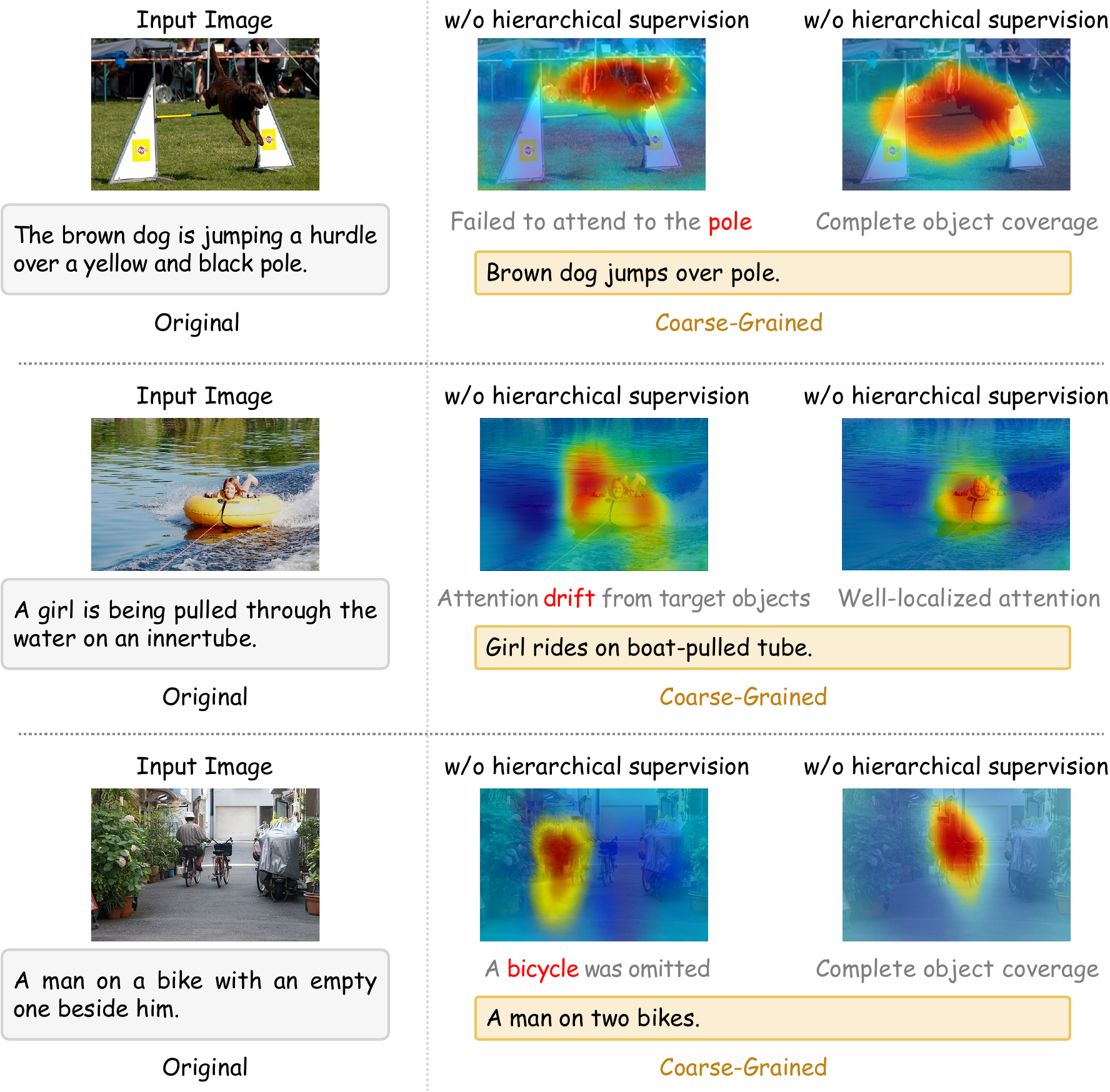}
    \vspace{-8pt}
    \caption{\textbf{Visual validation of Fine-to-Coarse Hierarchical Supervision (HS).} HS prevents semantic drift in coarse-grained representations. Without HS (middle), attention maps show common failures: missing key objects (second bicycle), poor component grounding (pole), or drift to irrelevant backgrounds. With HS (right), fine-level supervision guides coarse models to maintain semantic consistency, producing well-localized attention that accurately reflects textual descriptions.}
    \label{CAM_hierarchical_supervision}
\end{figure}

\section{Visualizations}
\label{app_vis}

\subsection{Granularity-Specific Retrieval Effectiveness.}
\label{appendix_granularity_specific}
The retrieval results in Fig.~\ref{coarse_also_yes_fig} validate our bidirectional connection mechanisms across granularity levels. Fine-grained representations consistently excel in both Text-to-Image and Image-to-Text tasks, resolving challenging cases requiring precise semantic understanding, such as distinguishing ``carrying a toy in its mouth'' from general dog activities. This capability stems from cascaded enhancement providing global contextual guidance, preventing attention from focusing solely on isolated details while capturing comprehensive information. Medium-grained representations outperform the single-granularity baseline while using fewer visual tokens. Coarse-grained representations achieve comparable performance despite using substantially fewer visual and textual tokens, enabled by hierarchical supervision that prevents semantic drift and preserves alignment quality with reduced representational capacity. These results confirm flexible performance-efficiency trade-offs while keeping semantic consistency across hierarchical structure.

\subsection{Visual Analysis of Cascaded Enhancement}
\label{appendix_vis_ce}
Fig.~\ref{CAM_cascaded_enhancement_fine} provides an analysis to visually validate the efficacy of our Coarse-to-Fine Cascaded Enhancement (CE) module. The comparison demonstrates that without CE (middle column), fine-grained attention struggles with context deficiency, failing to accurately ground complex descriptions involving multiple entities or specific attributes. For instance, it cannot disambiguate the "man reading a newspaper" from the one "smoking a cigarette," nor can it precisely locate the "striped shirt" or the "orange glasses." Conversely, by injecting global contextual guidance, our CE module (right column) resolves these ambiguities, enabling precise, contextually-aware alignment. The resulting attention maps successfully disentangle parallel actions and ground fine-grained attributes to their corresponding image regions. This visual evidence substantiates our claim that CE is crucial for addressing the limitations of isolated fine-grained processing, enabling robust alignment for complex, multi-faceted image-text pairs.

\subsection{Visual Analysis of Hierarchical Supervision}
\label{appendix_vis_hs}
Our Fine-to-Coarse Hierarchical Supervision (HS) plays a crucial role in mitigating semantic drift at coarser granularity levels, as visually validated in Fig.~\ref{CAM_hierarchical_supervision}. Without HS, coarse-grained models trained on simplified text-image pairs often produce flawed alignments; the attention may drift to background noise (e.g., the water wake instead of the girl), omit less salient objects mentioned in the text (e.g., the second bicycle), or fail to ground all relevant components (e.g., ignoring the pole). These inconsistencies arise because coarser levels are optimized in isolation with ambiguous supervision. Our HS mechanism addresses this by using the finest-level alignment as a teacher distribution to regularize the learning process across the hierarchy. As demonstrated in the right column, this forces the coarse-grained representations to maintain semantic consistency, resulting in well-localized attention and complete object coverage that correctly reflects the underlying semantics.

\end{document}